\PassOptionsToPackage{table}{xcolor}
\documentclass{article}

\usepackage{iclr2027_conference,times}

\usepackage[utf8]{inputenc}
\usepackage[T1]{fontenc}
\usepackage{hyperref}
\usepackage{url}
\usepackage{booktabs}
\usepackage{amsfonts}
\usepackage{nicefrac}
\usepackage{microtype}
\usepackage{xcolor}
\usepackage{amsmath}
\usepackage{amssymb}
\usepackage{amsthm}
\usepackage{mathtools}
\usepackage{bm}
\usepackage{graphicx}
\usepackage[table]{xcolor}
\definecolor{aliceblue}{RGB}{240,248,255}
\iclrfinalcopy

\newcommand{\vx}{\bm{x}}
\newcommand{\vz}{\bm{z}}                                  \newcommand{\vh}{\bm{h}}                                  \newcommand{\vbenc}{\bm{b}_{\text{enc}}}
\newcommand{\vbdec}{\bm{b}_{\text{dec}}}                  \newcommand{\vxhat}{\hat{\bm{x}}}

\newcommand{\vtheta}{\bm{\theta}}
\newcommand{\mW}{\bm{W}}
\newcommand{\mWenc}{\bm{W}_{\text{enc}}}
\newcommand{\mWdec}{\bm{W}_{\text{dec}}}
\usepackage{multirow}
\usepackage{algorithm}
\usepackage{algpseudocode}
\newtheorem{theorem}{Theorem}

\newtheorem{assumption}{Assumption}

\newcommand{\ones}{\bm{1}}

\newcommand{\mcO}{\mathcal{O}}

\def \bb{\bm{b}}

\def \ab{\bm{a}}

\def \xb{\bm{x}}

\def \Wb{\bm{W}}

\usepackage{pifont}

\newcommand{\vertii}[1]{{\left\vert\left\vert #1
    \right\vert\right\vert}}

\title{Rational Sparse Autoencoder\thanks{Version history: the Rational Sparse Autoencoder was originally formulated by N.~Yin and Y.~Yu in version~1 of this preprint (June 2026). In this version~(v2), M.~Tang and A.~Townsend joined as co-authors and contributed to the theoretical analysis.}}

\author{%
  \normalfont
  \begin{tabular}[t]{@{}p{0.45\textwidth}@{}p{0.45\textwidth}@{}}
  \textbf{Naiyu Yin}            & \textbf{Maosen Tang} \\
  Department of Mathematics     & Center for Applied Mathematics \\
  Lehigh University             & Cornell University \\
  Bethlehem, PA 18015           & Ithaca, NY 14853 \\
  \texttt{nay224@lehigh.edu}    & \texttt{mt872@cornell.edu} \\[1.2em]
  \textbf{Alex Townsend}        & \textbf{Yue Yu} \\
  Department of Mathematics     & Department of Mathematics \\
  Cornell University            & Lehigh University \\
  Ithaca, NY 14853              & Bethlehem, PA 18015 \\
  \texttt{townsend@cornell.edu} & \texttt{yuy214@lehigh.edu}
  \end{tabular}%
}

\begin{document}

\maketitle
\lhead{Preprint}

\begin{abstract}
  Sparse autoencoders (SAEs) are standard tools for mechanistic interpretability, but current SAE families are constrained by fixed encoder nonlinearities such as ReLU, JumpReLU, and TopK. This hard-codes a particular sparsity mechanism into the model and can distort the
  reconstruction-versus-sparsity trade-off. We introduce the \emph{Rational Sparse Autoencoder} (RSAE), which replaces the fixed encoder activation with a trainable rational function. The design is motivated by a reconstruction advantage that we prove for RSAE at matched expected sparsity, with guaranteed non-inferiority to ReLU and, under stated assumptions, provable improvements over JumpReLU and TopK. The model is then implemented through a two-stage pipeline: an
  initialization procedure that copies the pre-trained
  baseline SAE weights and calibrates the scale parameters along with the rational coefficients; followed by a fine-tuning step under a sparsity-regularized reconstruction objective that holds $\ell_0$ at the teacher's value. Consistent with this theory, on residual-stream activations of three open-weight language models and across all three baseline activation families, the RSAE \emph{improves} reconstruction fidelity after the fine-tuning step at the same level of sparsity. These gains are consistent across host language models, across baseline activation families, and across the full range of baseline sparsity we tested, while the upgrade itself adds only $20$ scalar parameters per autoencoder
  and one training run takes between about $4$ and $25$ minutes on a single consumer GPU.
\end{abstract}

\section{Introduction}
Sparse autoencoders (SAEs) have become a central tool for
mechanistic interpretability of large language models, decomposing
the internal activations of a transformer into a sparse linear
combination of an overcomplete dictionary of monosemantic feature
directions
\citep{bricken2023monosemanticity, huben2024sparse,
gao2025scaling, rajamanoharan2024improving, rajamanoharan2024jumprelu}.
Despite rapid progress on training procedures and evaluation suites,
the widely used released baselines considered here pair the same affine
encoder with one of three non-smooth activation primitives:
$\mathrm{ReLU}$ regularized by an $\ell_1$ penalty
\citep{bricken2023monosemanticity, bloom2024gpt2sae},
$\mathrm{TopK}$ \citep{gao2025scaling}, or $\mathrm{JumpReLU}$
\citep{rajamanoharan2024jumprelu}.  Each of these primitives carries
well-documented pathologies.  The $\ell_1$-regularized $\mathrm{ReLU}$
SAE suffers from magnitude shrinkage of active features and a
persistent population of dead latents
\citep{taggart2024prolu, rajamanoharan2024improving, gao2025scaling}.
$\mathrm{TopK}$ replaces the soft penalty with a hard cardinality
constraint that breaks gradient flow through inactive features and
relies on auxiliary revival losses to mitigate dead features.
$\mathrm{JumpReLU}$ inserts a learnable per-feature threshold but
requires a continuous-relaxation surrogate for back-propagation
through its indicator gate.

In this work we propose to replace these pre-selected activation functions with a trainable one, which we call the \emph{Rational Sparse Autoencoder} (RSAE). In particular,
the encoder activation in the RSAE is a learnable rational
function applied element-wise to the affine pre-activation, with
learnable input/output scales $(C_{\mathrm{in}}, C_{\mathrm{out}})$
that map the per-feature pre-activation distribution into a bounded
interval.  This architecture is supported by a new analysis: we provide reconstruction guarantees against
all three baseline activations, showing that the RSAE exactly represents every ReLU SAE,
preserving its code and reconstruction and yielding same-budget
non-inferiority (Appendix~\ref{appx:relu-containment}). Under Stability and
Nonlinear-calibration persistence, shared rational calibration gives an
$\Omega(k)$ reconstruction improvement over JumpReLU with per-feature
thresholds while preserving its support
(Appendix~\ref{appx:jumprelu-advantage}). For fixed raw $\mathrm{TopK}$, the
same calibration argument applies after a support replacement that costs
$O(\sqrt{k})$, so the gain dominates for sufficiently large sparse budgets
(Appendix~\ref{appx:general-theory}). Each comparison preserves the decoder
and latent width and matches the expected $\ell_0=k$.

For practitioners, we propose a two-step RSAE training algorithm. During
the initialization procedure, we copy the pre-trained baseline SAE
weights verbatim and
calibrate the scale parameters and the coefficients to the baseline's pre-activation
distribution. During the fine-tuning procedure, we unfreeze all
parameters and minimize a sparsity-regularized reconstruction objective that
holds the expected $\ell_0$ at the teacher's value. A single low-degree rational
of type $(9,8)$ reproduces a ReLU teacher exactly and a JumpReLU or TopK teacher
to within $1$--$3\%$ of its reconstruction error before fine-tuning. At the SAE level, we evaluate the RSAE on
residual-stream activations of three open-weight language models
spanning a range of model sizes and against all three baseline
activation families, supporting our central claim:
at matched $\ell_0$, the RSAE improves on its teacher in reconstruction error,
cross-entropy degradation, loss recovered and, in almost all settings, the
fraction of alive features, across host language models and baseline
activation families. A control that freezes the rational shows that most of the
gain on ReLU teachers comes from the trainable activation itself. These reconstruction fidelity gains are
observed across the full range of baseline sparsity we tested, and their relative size grows with $\ell_0$. A sparse-probing sanity check (Appendix~\ref{appx:probing}) shows that they do not come at the cost of probe accuracy. Each RSAE adds only $20$
scalar parameters to its teacher and is trained in at most about $25$ minutes on a single
consumer GPU.

\textbf{Contributions.}
We introduce \emph{RSAE}, a new sparse autoencoder built on a
trainable activation and supported by theoretical guarantees on reconstruction fidelity. We prove ReLU non-inferiority and,
under Stability and Nonlinear-calibration persistence, reconstruction
advantages over JumpReLU with per-feature thresholds and over fixed raw
$\mathrm{TopK}$ for sufficiently large sparse budgets, at matched expected
$\ell_0$. Practically, to implement this upgrading strategy, we propose a two-step
RSAE training algorithm: an initialization procedure that copies the
pre-trained baseline SAE weights, followed by a fine-tuning procedure that unfreezes all parameters
under a sparsity-regularized reconstruction objective at the teacher's $\ell_0$.
We show empirically that the RSAE achieves better fidelity at
matched sparsity with negligible parameter and runtime overhead: it improves
on its teacher in reconstruction error, cross-entropy degradation and loss
recovered across host language models, baseline activation families, and
baseline sparsity levels.

\section{Preliminaries and Related Work}
\label{sec:related-work}

\textbf{Sparse Autoencoders (SAEs)}
decompose a language model's internal activations $\vx \in \mathbb{R}^{d_{\mathrm{in}}}$ into a sparse linear combination of an overcomplete dictionary of $d_{\mathrm{sae}} \gg d_{\mathrm{in}}$ feature directions $\vz\in \mathbb{R}^{d_{\mathrm{sae}}}$. They follow a skeleton with a pair of encoder and decoder functions $(f,g)$ defined by:
\begin{equation}
\text{Encoder: } \vz = f(\vx):=\phi(\mWenc\,(\vx-\vbdec) + \vbenc),
\quad
\text{Decoder: } \vxhat = g(\vz):= \mWdec\,\vz + \vbdec.
\label{eq:sae-skel}
\end{equation}
We write $\vh \coloneqq \mWenc(\vx - \vbdec) + \vbenc$ for the pre-activation, so that $\vz = \phi(\vh)$.
{Here, columns of $\mWdec$ represent decoder dictionary directions used to reconstruct $\vx$ from the sparse code $\vz$ with unit $\ell_2$-norm.} The weights in encoder/decoder functions are optimized using a loss function of the form:
{\small\begin{equation}
     \mathcal{L}(\mW)=\mathbb{E}_{\vx \sim \mathcal{D}}
    \Bigl[\,
      \bigl\| \vx - \vxhat(\vx;\, \mW) \bigr\|_2^{2}
      \;+\;
      \lambda\,
      S\bigl( \vz(\vx;\, \mW) \bigr)
    \,\Bigr],\quad \mW:=\{\mWenc,\mWdec,\vbenc,\vbdec\},
    \label{eq:sae-objective}
\end{equation}}
where $S$ is a function that penalizes non-sparse decompositions with a tunable sparsity coefficient $\lambda$.

There are two objectives in the SAE encoder: sparsity, meaning that only a few elements of the dictionary are necessary, and faithfulness, meaning that the reconstructed $\vxhat$ is close to the original $\vx$. To achieve a good balance between these two objectives, three major SAE activations were proposed, which differ in the encoder activation $\phi$ and the sparsity mechanism $S$ imposed on $\vz$. The \textit{ReLU SAE}~\citep{bricken2023monosemanticity, bloom2024gpt2sae} sets $\phi = \mathrm{ReLU}$ and imposes sparsity through an explicit $\ell_1$ penalty $S(\vz):=\vertii{\vz}_1$. In the original \textit{ReLU SAE}, the soft $\ell_1$ penalty leads to a magnitude shrinkage of active features and causes loss of reconstruction fidelity~\citep{taggart2024prolu, rajamanoharan2024improving, gao2025scaling}. \textit{TopK SAE}~\citep{gao2025scaling} then proposes to replace the soft penalty with a hard top-$k$ selection $\vz = \mathrm{TopK}_k(\vh)$ that yields exact $\ell_0 = k$, and the \textit{JumpReLU SAE}~\citep{rajamanoharan2024jumprelu} penalizes $\ell_0$ directly, using straight-through estimators to train through the discontinuity, and inserts a learnable per-feature threshold $\theta_j > 0$ in the activation function $\phi$ by setting $\phi(\vh) = \vh\odot H(\vh-\vtheta)$, where $H$ is the Heaviside function satisfying $H(z)=0$ if $z\leq 0$ and $H(z)=1$ elsewhere.
Orthogonal to the development in encoder activations, \emph{Matryoshka SAEs}~\citep{bussmann2025learning} reorganize the decoder into nested-prefix dictionaries, and \emph{data-free SAEs}~\citep{laptev2025analyze} fit dictionaries directly from model weights without streaming activations.

While related variants such as the \emph{Gated SAE}~\citep{rajamanoharan2024improving}, ProLU~\citep{taggart2024prolu}, \emph{BatchTopK SAE}~\citep{bussmann2024batchtopk}, and end-to-end SAE training~\citep{braun2024identifying} modify thresholds, gates, batch-level sparsity, or the training objective, our theoretical results and empirical comparisons focus on the widely used released baselines considered here---ReLU, JumpReLU, and TopK SAEs---whose encoder nonlinearities are fixed functional forms with sparsity controlled by a penalty coefficient, learned threshold, or cardinality budget rather than by a trainable rational activation.
Through a trainable activation architecture supported by theoretical guarantees, our RSAE provides a drop-in modification for pre-trained SAEs (teacher models): at matched expected $\ell_0$, it is never worse than ReLU and, under stated assumptions, strictly improves reconstruction fidelity over JumpReLU and fixed raw $\mathrm{TopK}$.

\textbf{Evaluation benchmarks and pretrained baselines.}
\emph{SAEBench}~\citep{karvonen2025saebench} provides matched pretrained ReLU, JumpReLU, and TopK SAEs across model sizes and a unified evaluation suite (covering reconstruction, sparsity, downstream performance, and interpretability metrics); we use the pre-trained ReLU, JumpReLU, and TopK SAEs for our Pythia-160m and Gemma-2-2B baselines, and its core evaluation protocol and sparse-probing harness in~\S\ref{sec:experiments}. For GPT-2 small we additionally use Bloom's \emph{gpt2-small-res-jb} release~\citep{bloom2024gpt2sae} as the ReLU baseline and OpenAI's v5 release~\citep{gao2025scaling} as the TopK baseline.

\textbf{Rational Neural Networks}\label{sec:bg-rational}
were built on the key theoretical advantages of rational functions in approximating non-smooth functions~\citep{newman1979approximation,telgarsky2017neural,beckermann2017singular,chen2018rational}. In~\citet{boulle2020rational}, a rational function is employed as a learnable replacement for $\mathrm{ReLU}$ or $\tanh$ in feed-forward networks for image classification tasks. Superior performance was also demonstrated in operator learning and PDE surrogates, where the spectral density of rational approximation accelerates convergence on smooth target operators~\citep{trimmel2022era}. In rational neural networks,
the standard activation function in a feed-forward layer is replaced with a trainable rational function {\small $\frac{P(t)}{Q(t)}$}. A naive learnable denominator {\small$Q(t)=1 + \sum_{j=1}^{q} b_j\,t^{j}$} can develop divergent poles during training when $Q$ approaches zero. To prevent this, \citet{molina2019pade} introduced the Pad\'e Activation Unit (PAU) with the \emph{safe} denominator {\small $Q(t)=1 + |\sum_{j=1}^{q} b_j\,t^{j}|\ge1$}, which guarantees pole-free rational activations; the variant {\small $1+\sum_{j=1}^{q}|b_j|\,|t|^{j}$} used in~\eqref{eq:rsae-act} is likewise bounded below by one~\citep{molina2019pade,delfosse2021safepade}.

While prior works have demonstrated the theoretical advantages of
rational activation functions against continuous activation functions such as
$\mathrm{ReLU}$, $\mathrm{GeLU}$, or $\tanh$~\citep{boulle2020rational,
molina2019pade, delfosse2021safepade, trimmel2022era, tang2026rational},
little discussion has focused on discontinuous SAE gates such as
$\mathrm{JumpReLU}$ and $\mathrm{TopK}$, or on approximation gains under fixed sparsity. Herein, we establish ReLU non-inferiority by exact representation and analyze nonlinear-calibration gains against JumpReLU and fixed raw $\mathrm{TopK}$. These SAE-level comparisons retain the full decoder Gram interaction and match the expected $\ell_0$ exactly.

\section{Rational Sparse Autoencoder}
\label{sec:rational-SAE}

\begingroup
The \emph{Rational Sparse Autoencoder} (RSAE) replaces the encoder activation
in~\eqref{eq:sae-skel} with a trainable rational function followed by a
non-negativity clamp. For pre-activations $\vh=\mWenc(\vx-\vbdec)+\vbenc$,\footnote{We follow the
input convention of the teacher: the SAEBench Standard (ReLU) and TopK SAEs
subtract $\vbdec$ before the encoder, whereas the SAEBench JumpReLU SAEs use
$\vh=\mWenc\vx+\vbenc$; the RSAE inherits whichever convention its teacher
uses.} the sparse code is
\begin{equation}
  \vz=\phi(\vh)
  =C_{\mathrm{out}}\left[
    r_{(\ab,\bb)}\!\left(\frac{\vh}{C_{\mathrm{in}}}\right)
  \right]_+,
  \qquad
  r_{(\ab,\bb)}(t)
  =\frac{P(t)}{Q(t)}
  =\frac{\sum_{i=0}^{p}a_i t^i}{1+\sum_{j=1}^{q}|b_j|\,|t|^j},
  \label{eq:rsae-act}
\end{equation}
where $[u]_+=\max(0,u)$ is applied element-wise and $p$ and $q$ are the
numerator and denominator degrees. The
denominator uses the safe (pole-free) form: $b_0=1$ and the absolute values
guarantee $Q(t)\ge 1$ for every $t$ and every choice of $\bb$, so
$r_{(\ab,\bb)}$ has no poles on the real line and can be initialized and
fine-tuned without any constraint on the coefficients. The rational
coefficients $(\ab,\bb)$ and the positive scalar scales
$C_{\mathrm{in}},C_{\mathrm{out}}$ (parameterized through their logarithms)
are trainable and shared across all features; the support of the sparse code,
$\{\vh:\vz>0\}$, is therefore determined solely by the sign of $P(\vh/C_{\mathrm{in}})$
and involves no threshold or gating parameter. The sparse code is decoded by
$\vxhat=\mWdec\vz+\vbdec$.

We compare reconstruction at the same expected support size
$k=\mathbb E\lVert z^{\mathrm{base}}\rVert_0$, where
$z^{\mathrm{base}}$ is the baseline code. Write
$\mathcal R(z)=\mathbb E\lVert\vx-\mWdec z-\vbdec\rVert_2^2$ for the
reconstruction risk. For ReLU, containment directly yields same-budget
non-inferiority.

\begin{theorem}[ReLU non-inferiority]
\label{thm:main-relu}
For $p\ge1$ and $q\ge0$, the infimum of RSAE reconstruction risk is no
greater than that of ReLU at any feasible expected $\ell_0$ budget, under
the same encoder and decoder parameter constraints. Non-inferiority also holds
for the regularized objective~\eqref{eq:rsae-objective}, for every
$\lambda\ge0$ and every regularizer that depends only on the code, including $\ell_1$.
\end{theorem}

For fixed raw TopK and JumpReLU with positive per-feature thresholds,
reconstruction improvements follow under the next two assumptions.

\begin{assumption}[Stability]
\label{ass:main-stability}
The mean squared active-feature magnitude is uniformly bounded.
For multiplicative perturbations defined by a bounded shared function
of the pre-activation, expected decoded energy is at most a fixed
constant times expected coefficient energy.
For TopK, a common threshold matching the expected count $k$ gives
count variance $O(k)$, and successive reconstruction gains in descending
score order have uniformly bounded moments of some order greater than
two about a common deterministic center.
\end{assumption}

\begin{assumption}[Nonlinear-calibration persistence]
\label{ass:main-persistent-calibration}
After optimal shared rescaling, the product of an active coefficient
and the residual's projection onto its decoder direction has a second
moment bounded away from zero and a bounded normalized fourth moment.
Its conditional mean given the scalar pre-activation captures a
nonvanishing fraction of this second moment.
\end{assumption}

The quantitative conditions are stated in Assumptions~\ref{ass:stability}
and~\ref{ass:persistent-calibration}.

\begin{theorem}[Same-budget reconstruction advantage]
\label{thm:main-reconstruction}
Under Assumptions~\ref{ass:main-stability}
and~\ref{ass:main-persistent-calibration}, in the settings of
Theorems~\ref{thm:general-topk-advantage} and~\ref{thm:jumprelu-advantage},
there are constants $c\ge0$ and $g>0$, depending only on the assumption
bounds, such that for every $\varepsilon>0$ there exists an RSAE
with a finite-degree shared rational activation without real poles, the
same decoder and latent width, and
$\mathbb E\lVert z^{\mathrm{Rat}}\rVert_0=k$, satisfying
\begin{equation}
  \mathcal R(z^{\mathrm{base}})-\mathcal R(z^{\mathrm{Rat}})
  \ge
  \begin{cases}
    gk-c\sqrt{k}-\varepsilon,
      & \mathrm{base}=\mathrm{TopK},\\
    gk-\varepsilon,
      & \mathrm{base}=\mathrm{JumpReLU}.
  \end{cases}
  \label{eq:rsae-reconstruction-guarantees}
\end{equation}
The encoder is unchanged for TopK; an affine reparameterization preserves
the per-feature support for JumpReLU.
\end{theorem}

Consequently, at the same expected sparsity, RSAE attains lower
reconstruction risk than JumpReLU and, for sufficiently large $k$,
fixed raw TopK. The constants $g$ and $c$ depend on the instance through the assumption bounds, so the bound identifies the regime that favors the RSAE (budgets $k$ that are large relative to the support-replacement cost) rather than a fixed absolute rate of improvement in $k$.

Proofs are given in Appendices~\ref{appx:relu-containment},
\ref{appx:general-theory}, and~\ref{appx:jumprelu-advantage}.
\endgroup

\section{Practical Algorithm}
\label{sec:proposed-algorithm}

The RSAE is constructed in two steps (Algorithm~\ref{alg:rsae}, Appendix~\ref{appx:implementation}): (i) an \emph{initialization procedure} that inherits the teacher's encoder and decoder, sets the rational coefficients $(\ab,\bb)$ in closed form so that the
rational activation reproduces the teacher's activation, and adapts the learnable scales
$(C_{\mathrm{in}},C_{\mathrm{out}})$ together with $(\ab,\bb)$ to the teacher's pre-activation
distribution; and (ii) a \emph{fine-tuning procedure} that jointly optimizes all parameters,
including the encoder and decoder weights, under a sparsity-regularized reconstruction objective.

\textbf{Step 1: RSAE Initialization Procedure.}
We first initialize the encoder and decoder with the pre-trained teacher weights
$\{\widetilde{\Wb}_{\mathrm{enc}},\widetilde{\bb}_{\mathrm{enc}},\widetilde{\Wb}_{\mathrm{dec}},\widetilde{\bb}_{\mathrm{dec}}\}$, and set the scales to $C_{\mathrm{in}}=C_{\mathrm{out}}=\operatorname{std}(\vh)$ on a calibration
batch. Two closed-form initializations of the rational coefficients are used, depending on the teacher's activation.

\emph{Identity-init (ReLU teachers).} We set the rational to the identity map,
$$\ab^\ast=(0,\,1,\,0,\,\dots,\,0),\qquad \bb^\ast=(1,\,0,\,\dots,\,0),$$
so that, with the non-negativity clamp of \eqref{eq:rsae-act} and $C_{\mathrm{in}}=C_{\mathrm{out}}$,
$r_{(\ab^\ast,\bb^\ast)}(\cdot)$ is exactly the ReLU activation and the RSAE coincides with the
teacher at initialization.

\emph{Knee-init (JumpReLU and TopK teachers).} A gated teacher computes
$\phi^{\mathrm{teacher}}_j(\vh)=h_j\,\ones[h_j>\tilde{\theta}_j]$ with a per-feature threshold
$\tilde{\theta}_j$ inherited from the teacher\footnote{JumpReLU teachers learn $\tilde{\theta}_j$ during training. A TopK teacher has no
threshold: it keeps the $k$ largest pre-activations per token. Its SAEBench release stores a scalar
inference threshold, the mean over tokens of the $k$-th largest pre-activation, with which the SAE can
be run as a JumpReLU-style gate at the same mean $\ell_0$; we use this value as
$\tilde{\theta}_j\equiv\tilde{\theta}$ for all features (estimating it on a calibration batch if a
checkpoint lacks it). Teachers themselves are always evaluated with exact top-$k$ selection.}. A shared rational cannot represent per-feature jump
heights, so we first absorb $\tilde{\theta}_j$ into the inherited weights by the function-preserving
rescaling
$\widetilde{\Wb}_{\mathrm{enc}}[:,j]\leftarrow\widetilde{\Wb}_{\mathrm{enc}}[:,j]/\tilde{\theta}_j$,
$\widetilde{b}_{\mathrm{enc},j}\leftarrow\widetilde{b}_{\mathrm{enc},j}/\tilde{\theta}_j-1$,
$\widetilde{\Wb}_{\mathrm{dec}}[j,:]\leftarrow\tilde{\theta}_j\widetilde{\Wb}_{\mathrm{dec}}[j,:]$
(Appendix~\ref{app:gauge}), after which every feature computes the same unit jump
$(v+1)\,\ones[v>0]$ of its rescaled pre-activation $v$. We then initialize the rational as the
lowest-degree knee that reproduces this jump,
$$
r_{(\ab^\ast,\bb^\ast)}(t)=\frac{a_1 t+a_2 t^2+a_3 t^3}{1+K^2 t^2},
$$
with $(a_1,a_2,a_3)$ fixed by three conditions: (i) the tail matches the jump,
$C_{\mathrm{out}}\,r(\frac{v}{C_{\mathrm{in}}})=v+1+\mcO(\frac{1}{v})$, giving
$a_3=K^2\frac{C_{\mathrm{in}}}{C_{\mathrm{out}}}$ and $a_2=\frac{K^2}{C_{\mathrm{out}}}$;
(ii) the knee has width $\approx\frac{1}{K}$, giving $a_1\ge \frac{K}{C_{\mathrm{out}}}$; and
(iii) $r(t)<0$ for all $t<0$, so that the clamp is exactly zero off the teacher's support, which holds
iff $a_2^2<4a_1a_3$. We take $a_1=\max\bigl(\frac{K}{C_{\mathrm{out}}},\,1.05\,\frac{a_2^2}{4a_3}\bigr)$ and
$K=30$ (JumpReLU) or $15$ (TopK).

With the initialized weights and coefficients, we adapt the rational activation to the teacher's
pre-activation distribution by minimizing
\begin{equation}
\small
(\widetilde{\ab},\widetilde{\bb},\widetilde{C}_{\mathrm{in}},\widetilde{C}_{\mathrm{out}})
\;=\;
\arg\min_{\ab,\bb,C_{\mathrm{in}},C_{\mathrm{out}}}\;
\bigl\|\phi(\vh;\ab,\bb,C_{\mathrm{in}},C_{\mathrm{out}})-\phi^{\mathrm{teacher}}(\vh)\bigr\|_2^2
\;+\;\mu\,S\bigl(\vz(\vh;\ab,\bb,C_{\mathrm{in}},C_{\mathrm{out}})\bigr),
\label{eq:init-distill}
\end{equation}
where $\phi(\vh;\ab,\bb,C_{\mathrm{in}},C_{\mathrm{out}})$ is computed via \eqref{eq:rsae-act} with the
encoder and decoder frozen, $\vh$ is the teacher's pre-activation, and $S(\cdot)$ is a differentiable
sparsity regularizer on the sparse code with coefficient $\mu\ge0$, which keeps the support of the rational
aligned with the teacher's while the first term fits the magnitudes. For gated teachers we run
\eqref{eq:init-distill} from both initializations and keep the better fit. Combined with the
inherited encoder and decoder weights, the learned coefficients and scales allow the RSAE to reproduce the teacher's output, up to the Step~1 approximation error.

\textbf{Step 2: RSAE Fine-Tuning Procedure.}
We initialize the RSAE encoder and decoder with the teacher weights
$\{\widetilde{\Wb}_{\mathrm{enc}},\widetilde{\bb}_{\mathrm{enc}},\widetilde{\Wb}_{\mathrm{dec}},\widetilde{\bb}_{\mathrm{dec}}\}$,
and the rational activation with the learned coefficients $(\widetilde{\ab},\widetilde{\bb})$ and
scales $(\widetilde{C}_{\mathrm{in}},\widetilde{C}_{\mathrm{out}})$. We then unfreeze all parameters
$\Theta\coloneqq\{\Wb_{\mathrm{enc}},\bb_{\mathrm{enc}},\Wb_{\mathrm{dec}},\bb_{\mathrm{dec}},\ab,\bb,C_{\mathrm{in}},C_{\mathrm{out}}\}$
and minimize the sparsity-regularized reconstruction objective
\begin{equation}
    \min_{\Theta}\;\;
    \mathbb{E}_{\vx\sim\mathcal{D}}
    \Bigl[\,
      \bigl\|\vx-\vxhat(\vx;\Theta)\bigr\|_2^{2}
    \Bigr]
    \;+\;
    \lambda\,S\bigl(\vz(\vx;\Theta)\bigr),
    \label{eq:rsae-objective}
\end{equation}
with sparsity coefficient $\lambda>0$. Because the RSAE has no threshold or gating parameters, its
support is determined solely by the root of the rational, so $S$ controls sparsity only through the
rational coefficients, the scales and the encoder. The choice of $S$ (a target-$\ell_0$ regularizer
in our experiments; an $\ell_1$ penalty recovers the classical objective), the calibration of its
target, and all optimization details are deferred to Appendix~\ref{app:training}.

\section{Empirical Results}\label{sec:experiments}

\providecommand{\pmsd}[1]{{\scriptsize\textcolor{gray}{$\pm$#1}}}
\begin{table}[t]
  \centering
  \caption{{\small\textbf{Main results:} teacher SAEs and the RSAEs fine-tuned from them (RSAE (\emph{teacher}-init)) on residual-stream activations of GPT-2 small (layer~6), Pythia-160m (layer~8) and Gemma-2-2B (layer~12). GPT-2 small teachers are the public releases: Bloom's \texttt{gpt2-small-res-jb} ReLU SAE and OpenAI's \texttt{gpt2-small-resid-post-v5-32k} TopK SAE; no JumpReLU SAE has been released for GPT-2 small. Pythia-160m and Gemma-2-2B teachers are SAEBench releases; results for all layers, dictionary widths $d_{\mathrm{sae}}$ and sparsity levels are reported in Tables~\ref{tab:gpt2}--\ref{tab:gpt2-layers}. Metrics: MSE, the mean per-token squared reconstruction error, multiplied by 100 for GPT-2 small and Pythia-160m and in raw units for Gemma-2-2B; $\ell_0$, the mean number of non-zero latents per token; alive, the fraction of alive latents in percent; $\Delta$CE, the cross-entropy degradation; LR, the loss recovered in percent. For every RSAE we run three independent fine-tuning trials with different random seeds (data order and stochastic training) and report the mean$\,\pm\,$standard deviation across the three trials; teacher rows are deterministic evaluations of the released checkpoints. \textbf{Bold} marks the strictly better value within each pair; $\ell_0$ is not bolded.}}
  \label{tab:main-results}
  \small
  \setlength{\tabcolsep}{6pt}
  \renewcommand{\arraystretch}{1.08}
  \begin{tabular}{l c ccccc}
  \toprule
  SAE & $d_{\mathrm{sae}}$ & $\ell_0$ & MSE$\downarrow$ & alive (\%)$\uparrow$ & $\Delta$CE$\downarrow$ & LR (\%)$\uparrow$ \\
  \midrule
  \multicolumn{7}{l}{\textit{GPT-2 small (layer 6)}} \\[1pt]
  ReLU teacher & $24{,}576$
    & 51 & 21.42 & \textbf{100.0} & 0.142 & 98.6 \\
  \rowcolor{aliceblue}
  RSAE (ReLU-init) &
    & 52\pmsd{0.3} & \textbf{20.78}\pmsd{0.001} & 92.0\pmsd{0.1} & \textbf{0.101}\pmsd{0.003} & \textbf{99.0}\pmsd{0.02} \\
  TopK teacher & $32{,}768$
    & 32 & 29.44 & \textbf{100.0} & 0.114 & 98.9 \\
  \rowcolor{aliceblue}
  RSAE (TopK-init) &
    & 32\pmsd{0.1} & \textbf{26.70}\pmsd{0.007} & 99.2\pmsd{0.1} & \textbf{0.096}\pmsd{0.000} & \textbf{99.0}\pmsd{0.01} \\
  \midrule
  \multicolumn{7}{l}{\textit{Pythia-160m (layer 8)}} \\[1pt]
  ReLU teacher & $2^{14}$
    & 501 & 1.26 & 74.9 & 0.114 & 98.5 \\
  \rowcolor{aliceblue}
  RSAE (ReLU-init) &
    & 502\pmsd{0.1} & \textbf{0.31}\pmsd{0.001} & \textbf{83.0}\pmsd{0.3} & \textbf{0.018}\pmsd{0.002} & \textbf{99.8}\pmsd{0.01} \\
  JumpReLU teacher & $2^{14}$
    & 314 & 0.60 & \textbf{26.1} & 0.036 & 99.5 \\
  \rowcolor{aliceblue}
  RSAE (JumpReLU-init) &
    & 314\pmsd{0.2} & \textbf{0.57}\pmsd{0.001} & 25.7\pmsd{0.0} & \textbf{0.035}\pmsd{0.002} & 99.5\pmsd{0.01} \\
  TopK teacher & $2^{12}$
    & 640 & 0.15 & 23.3 & 0.010 & 99.9 \\
  \rowcolor{aliceblue}
  RSAE (TopK-init) &
    & 639\pmsd{0.1} & \textbf{0.13}\pmsd{0.003} & \textbf{100.0}\pmsd{0.0} & \textbf{0.007}\pmsd{0.001} & 99.9\pmsd{0.02} \\
  \midrule
  \multicolumn{7}{l}{\textit{Gemma-2-2B (layer 12)}} \\[1pt]
  ReLU teacher & $2^{16}$
    & 1272 & 0.743 & \textbf{75.2} & 0.031 & 99.7 \\
  \rowcolor{aliceblue}
  RSAE (ReLU-init) &
    & 1274\pmsd{0.1} & \textbf{0.229}\pmsd{0.001} & 72.1\pmsd{0.2} & \textbf{0.016}\pmsd{0.001} & \textbf{99.8}\pmsd{0.01} \\
  JumpReLU teacher & $2^{14}$
    & 712 & 0.468 & \textbf{96.2} & 0.016 & 99.8 \\
  \rowcolor{aliceblue}
  RSAE (JumpReLU-init) &
    & 717\pmsd{0.4} & \textbf{0.464}\pmsd{0.007} & 96.1\pmsd{0.0} & 0.016\pmsd{0.002} & 99.8\pmsd{0.01} \\
  TopK teacher & $2^{12}$
    & 640 & 0.793 & 97.3 & 0.031 & 99.7 \\
  \rowcolor{aliceblue}
  RSAE (TopK-init) &
    & 641\pmsd{0.4} & \textbf{0.754}\pmsd{0.004} & \textbf{97.7}\pmsd{0.0} & 0.031\pmsd{0.001} & 99.7\pmsd{0.01} \\
  \bottomrule
  \end{tabular}\vspace{-0.1in}
\end{table}

\begin{figure*}[t]
  \centering
  \includegraphics[width=\textwidth]{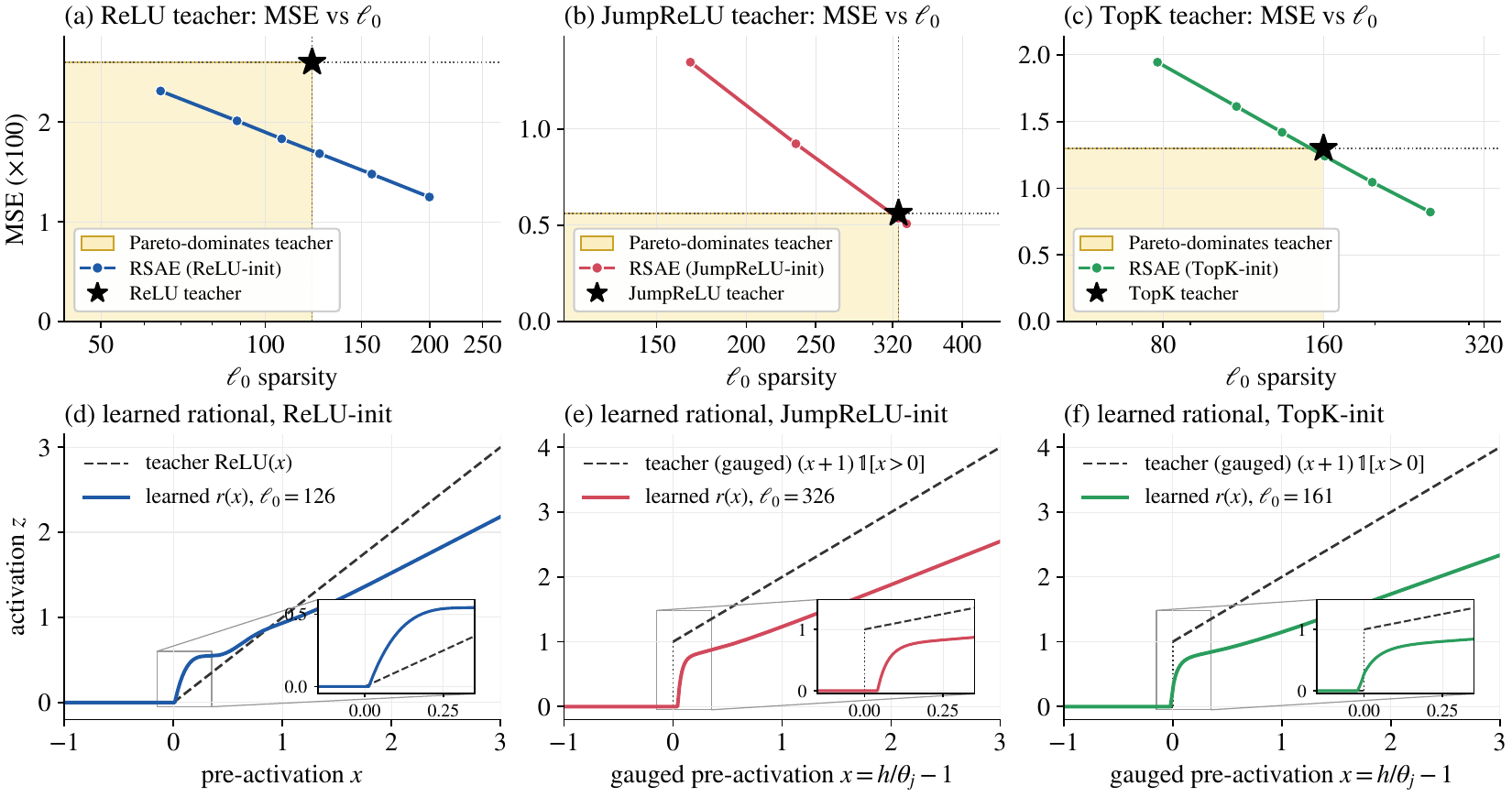}
  \vspace{-0.2in}
  \caption{\textbf{Pareto fronts and learned activations of RSAE} (Pythia-160m-deduped, layer~8, $d_\mathrm{sae}=2^{12}$, SAEBench core evaluation).
  \textbf{Top (a--c):} reconstruction MSE ($\times 100$) versus $\ell_0$ for one SAEBench teacher per family (star) and RSAEs fine-tuned from it at several target sparsities (circles).
  The shaded box marks the region that Pareto-dominates the teacher, i.e.\ lower MSE at equal or lower $\ell_0$.
  \textbf{Bottom (d--f):} the learned shared rational activation $r(x)$ of type $(9,8)$ for the RSAE closest to the teacher's $\ell_0$, against the teacher's activation (dashed).
  For JumpReLU and TopK teachers the pre-activation is threshold-gauged, $x=h/\theta_j-1$, so every teacher latent becomes the unit jump $(x+1)\,\ones[x>0]$.
  Insets zoom in around $x=0$.
  The vertical scale of $r$ trades off against the decoder norms, so only its shape should be compared with the teacher.}
  \label{fig:pareto-rationals}
  \vspace{-0.1in}
\end{figure*}

\begin{figure*}[t]
  \centering
  \includegraphics[width=\textwidth]{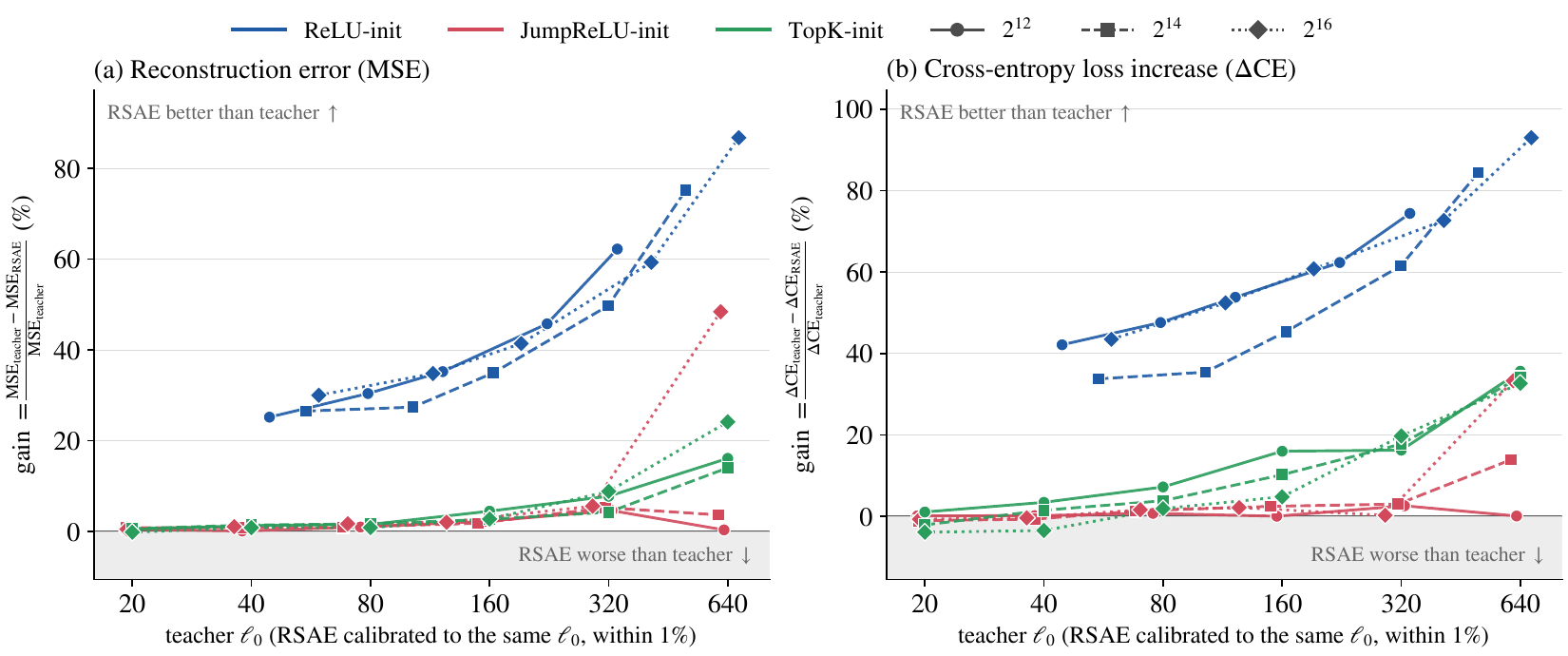}
  \vspace{-0.2in}
  \caption{\textbf{RSAE improvement over its teacher SAE as a function of sparsity} (Pythia-160m-deduped, layer~8, $d_\mathrm{sae}=\{2^{12}, 2^{14}, 2^{16}\}$, SAEBench core evaluation).
  The $x$-axis is the teacher's $\ell_0$ on a log scale.
  \textbf{(a)}~Relative reduction in reconstruction error, $\frac{\mathrm{MSE}_\mathrm{teacher}-\mathrm{MSE}_\mathrm{RSAE}}{\mathrm{MSE}_\mathrm{teacher}}$.
  \textbf{(b)}~Relative reduction in cross-entropy loss increase, $\frac{\Delta\mathrm{CE}_\mathrm{teacher}-\Delta\mathrm{CE}_\mathrm{RSAE}}{\Delta\mathrm{CE}_\mathrm{teacher}}$.}
  \label{fig:rsae-gain-vs-l0}
  \vspace{-0.2in}
\end{figure*}

\noindent\textbf{Models, SAEs, and Evaluation Metrics.} We evaluate the RSAE on residual-stream
activations from three open-weight language models of various sizes: \textbf{GPT-2
small}, \textbf{Pythia-160m-deduped}, and
\textbf{Gemma-2-2B}.  Teacher
SAEs are taken from publicly released checkpoints: GPT-2 small from
Bloom's \texttt{gpt2-small-res-jb}~\citep{bloom2024gpt2sae} (ReLU) and the OpenAI-v5~\citep{gao2025scaling}
\texttt{gpt2-small-resid-post-v5-32k} (TopK); Pythia-160m and Gemma-2-2B
from SAEBench~\citep{karvonen2025saebench} for ReLU, JumpReLU, and TopK. Following the SAEBench core
evaluation protocol~\citep{karvonen2025saebench}, we evaluate our RSAE against the SAE baselines in
terms of five metrics: (1) the reconstruction error, the mean over tokens of
$\|\vx-\vxhat\|_2^2$; (2) sparsity, the mean number of non-zero latents per token,
$\ell_0=\|\vz\|_0$; (3) the fraction of alive latents, i.e.\ latents that fire at least once on the
evaluation set; (4) the cross-entropy degradation when the reconstruction replaces the residual
stream, $\Delta\mathrm{CE}=\mathrm{CE}_{\vxhat}-\mathrm{CE}_{\vx}$ (lower is better); and (5) the
loss-recovered fraction
$\mathrm{LR}=\frac{\mathrm{CE}_{\mathrm{zero}}-\mathrm{CE}_{\vxhat}}{\mathrm{CE}_{\mathrm{zero}}-\mathrm{CE}_{\vx}}$
relative to zero ablation (higher is better). Please refer to Appendix~\ref{appx:implementation}
for the evaluation data and further implementation details.

\noindent\textbf{Choice of degrees.}
All experiments use a single rational type $(p,q)=(9,8)$ for every teacher family. The choice is not
critical: the identity initialization reproduces ReLU exactly at any degree, and the knee
initialization only needs a cubic numerator and a quadratic denominator, so higher degrees add
capacity that is used only during fine-tuning, where the learned coefficients beyond degree three
remain small. In an ablation on Pythia-160m (ReLU teacher, $d_{\mathrm{sae}}=2^{12}$,
$\ell_0\approx160$) the types $(3,2)$ and $(9,8)$ ended within $1\%$ of each other on every metric
after fine-tuning; $(9,8)$ is the largest type we found to train stably under the pole-free
denominator of \eqref{eq:rsae-act}, and we adopt it throughout so that one architecture covers all
teachers. The activation therefore adds only $18$ coefficients and two scales to the teacher's
parameter count.

\noindent\textbf{Main results.}
In all eight (model, teacher) pairs of Table~\ref{tab:main-results}, the RSAE
lowers the reconstruction error at matched $\ell_0$ (within $2\%$ of the
teacher's). It also lowers $\Delta$CE in six pairs and ties it in the two
Gemma-2-2B gated pairs, and it never lowers LR. Over all $51$ SAEBench
teachers on Pythia-160m (Table~\ref{tab:saebench-teachers-widths},
Appendix~\ref{appx:detailed-results}) the RSAE lowers the reconstruction error
of $47$, ties three and raises it for one (TopK at $d_\mathrm{sae}=2^{16}$,
$\ell_0\approx20$: $2.35\to2.36$); on MSE, $\Delta$CE and LR together it wins
$116$ of $153$ comparisons, ties $26$ and loses $11$, all at $\ell_0\le40$. On
Gemma-2-2B (Table~\ref{tab:gemma-full}) it lowers the reconstruction error of
all $18$ ReLU teachers and of $27$ of the $36$ gated ones; the $7$ increases
are at most $0.018$.

\noindent\textbf{(C1) RSAE improves ReLU SAEs by a large margin.}
The RSAE reduces the reconstruction error by $3.0\%$, $75.4\%$ and $69.2\%$
and $\Delta$CE by $28.9\%$, $84.2\%$ and $48.4\%$ on GPT-2 small, Pythia-160m
and Gemma-2-2B (Table~\ref{tab:main-results}). At matched $\ell_0$ it closes about $80\%$ of the error gap
between the ReLU and the TopK teacher ($d_\mathrm{sae}=2^{12}$, $\ell_0\approx40$ and $80$), and in
Figure~\ref{fig:pareto-rationals}(a) it reaches the teacher's error with about
half of its active latents. Theorem~\ref{thm:main-relu} guarantees only that
the RSAE is never worse, since the identity rational reproduces the teacher
exactly and fine-tuning starts from it; the size of the gain comes from the
response the RSAE learns. Figure~\ref{fig:pareto-rationals}(d) shows that,
starting from the identity, it develops a smooth jump at the origin, the
correction that JumpReLU and ProLU hard-code against $\ell_1$
shrinkage~\citep{rajamanoharan2024jumprelu,taggart2024prolu}. The control
below confirms that this gain comes from the trainable activation rather than
from the additional fine-tuning.

\noindent\textbf{(C2) RSAE reproduces a gate and improves most JumpReLU and TopK SAEs.}
The knee initialization reproduces a gated teacher to within $1$--$3\%$ of its
error, and the fine-tuned rational keeps a sharp gate at the teacher's gauged
threshold (Figure~\ref{fig:pareto-rationals}(e,f)). The RSAE then lowers MSE by
$0.9$--$13.3\%$ in all five gated pairs of Table~\ref{tab:main-results} and
$\Delta$CE by $2.8$--$30.0\%$ on GPT-2 small and Pythia-160m (it is tied on
Gemma-2-2B), and it lowers MSE for $32$ of the $36$ gated Pythia-160m teachers
of Table~\ref{tab:saebench-teachers-widths}, ties three and raises it for one. The
RSAE follows the structure of the construction behind
Theorem~\ref{thm:main-reconstruction}: its support is a threshold support (the
positive set of $P$), and the threshold re-gauging is the affine
reparameterization of Theorem~\ref{thm:jumprelu-advantage}. On Pythia-160m and Gemma-2-2B the gains are an order
of magnitude smaller than for ReLU, since gated teachers leave less systematic
error. The comparison with
the theory is qualitative, as its constants are teacher-dependent.

\begin{table}[t]
  \centering
  \caption{\textbf{Isolating the rational activation} (Pythia-160m-deduped, layer~8, $d_{\mathrm{sae}}=2^{12}$, SAEBench core evaluation). For each teacher, \emph{frozen rational} runs exactly the RSAE pipeline, i.e.\ the same initialization, data, fine-tuning steps, schedule, auxiliary loss and $\ell_0$ calibration, but keeps the rational coefficients and scales $(\ab,\bb,C_{\mathrm{in}},C_{\mathrm{out}})$ fixed at their Step-1 values, so only the encoder and decoder are fine-tuned; for a ReLU teacher this is a ReLU SAE fine-tuned with our objective. All models are calibrated to the teacher's $\ell_0$ within $1.5\%$. MSE is multiplied by 100; alive and LR are in percent.}
  \label{tab:isolate-rational}
  \small
  \setlength{\tabcolsep}{6pt}
  \begin{tabular}{l l ccccc}
  \toprule
  Teacher & Model & $\ell_0$ & MSE$\downarrow$ & alive$\uparrow$ & $\Delta$CE$\downarrow$ & LR$\uparrow$ \\
  \midrule
  \multirow{3}{*}{ReLU}     & teacher         & 122 & 2.60 & 76.2 & 0.404 & 94.5 \\
                            & frozen rational & 122 & 2.54 & 79.7 & 0.358 & 94.9 \\
  \rowcolor{aliceblue} \cellcolor{white} & RSAE & 121 & \textbf{1.69} & \textbf{99.9} & \textbf{0.187} & \textbf{97.5} \\
  \midrule
  \multirow{3}{*}{JumpReLU} & teacher         & 155 & 1.34 & 81.0 & \textbf{0.116} & \textbf{98.4} \\
                            & frozen rational & 155 & 1.33 & 99.9 & 0.118 & \textbf{98.4} \\
  \rowcolor{aliceblue} \cellcolor{white} & RSAE & 157 & \textbf{1.31} & \textbf{100.0} & \textbf{0.116} & \textbf{98.4} \\
  \midrule
  \multirow{3}{*}{TopK}     & teacher         & 160 & 1.30 & 99.5 & 0.125 & 98.3 \\
                            & frozen rational & 161 & 1.29 & \textbf{99.9} & 0.120 & 98.3 \\
  \rowcolor{aliceblue} \cellcolor{white} & RSAE & 161 & \textbf{1.24} & \textbf{99.9} & \textbf{0.105} & \textbf{98.6} \\
  \bottomrule
  \end{tabular}
\end{table}

\noindent\textbf{Isolating the rational activation.}
Because the RSAE is fine-tuned from its teacher, part of its gain could come
from the additional training, the new objective or the auxiliary loss rather
than from the rational. Table~\ref{tab:isolate-rational} separates these
effects on Pythia-160m by running the identical pipeline with the rational
frozen at its Step-1 values. For the ReLU teacher this control, a ReLU SAE
fine-tuned with our objective, recovers only a small part of the gain: it
lowers MSE from $2.60$ to $2.54$ and raises the alive fraction from $76.2\%$ to
$79.7\%$, whereas training the rational lowers MSE to $1.69$ and raises alive
to $99.9\%$. The trainable activation therefore accounts for $93\%$ of the
reconstruction gain, $79\%$ of the $\Delta$CE gain and $85\%$ of the
alive-fraction gain. For gated teachers the frozen knee already matches the
teacher (MSE $1.33$ versus $1.34$ for JumpReLU and $1.29$ versus $1.30$ for
TopK), and training the rational adds a further $0.02$--$0.05$ and lowers
$\Delta$CE. The rational's contribution beyond reproducing the gate is
therefore modest, and the large gains are concentrated on ReLU
teachers.

\noindent\textbf{Ablation studies on sparsity and width.}
Figure~\ref{fig:rsae-gain-vs-l0} and Table~\ref{tab:saebench-teachers-widths}
show that the gains are consistent across sparsity and width: the RSAE
raises the reconstruction error in only one of the $51$ (sparsity, width,
family) settings, by $0.01$, improves MSE, $\Delta$CE and LR for every ReLU
teacher, and gives the largest gains for ReLU teachers at every width; the
alive fraction can fall below the teacher's, mostly for ReLU at $d_\mathrm{sae}=2^{16}$.

\noindent\emph{Larger budgets $k$ yield larger improvements.} Figure~\ref{fig:rsae-gain-vs-l0}
shows that the RSAE's advantage grows as the sparsity budget $k=\ell_0$ increases,
as anticipated by Theorem~\ref{thm:main-reconstruction}, where the calibration
gain scales as $gk$ while the TopK support-replacement cost grows only as
$c\sqrt{k}$. At $d_\mathrm{sae}=2^{12}$, for ReLU teachers the MSE gain rises from $25\%$ at
$\ell_0\approx40$ to $63\%$ at $\ell_0\approx336$ and the $\Delta$CE gain from
$42\%$ to $74\%$; for TopK teachers the MSE gain rises from $0.6\%$ at
$\ell_0\approx20$ to $13.3\%$ at $\ell_0\approx640$ and the $\Delta$CE gain from
$1.2\%$ to $30.0\%$. Table~\ref{tab:saebench-teachers-widths} confirms the trend
at the larger widths: the ReLU MSE gain rises monotonically at every width, to
$75\%$ and $87\%$ at $d_\mathrm{sae}=2^{14}$ and $2^{16}$, and the TopK MSE gain
rises to $13.3\%$ and $25.0\%$ at $\ell_0\approx640$, with a single dip at
$d_\mathrm{sae}=2^{14}$. JumpReLU gains stay below $6\%$ up to $\ell_0\approx320$; at
$\ell_0\approx640$ they are zero at $d_\mathrm{sae}=2^{12}$ and $2^{14}$ and
$50\%$ at $2^{16}$. The growth is in the share of the remaining error that the
RSAE removes; in absolute terms the error removed is roughly constant, because
denser teachers leave a smaller residual for the rational to calibrate.
At $\ell_0\approx640$ the gated teachers leave almost no headroom: all reach
$99.9\%$ loss recovered with $\Delta\mathrm{CE}\le0.010$, so $9$ of the $26$
ties in Table~\ref{tab:saebench-teachers-widths} occur here, and MSE
improves by at most $0.04\times10^{-2}$ except for the JumpReLU teacher at
$2^{16}$ ($1.7\%$ alive latents). ReLU teachers, by contrast,
remain far from saturation at their densest setting, where the RSAE's gains
are largest.

\begin{table}[hpt]
  \centering
  \caption{{\small\textbf{Wall-clock runtime} of the RSAE pipeline per model, averaged across teacher SAEs, measured on a single NVIDIA RTX 5090 (32\,GB). The Init Procedure comprises weight inheritance, the closed-form initialization and the distillation of the rational coefficients and scales (skipped for ReLU teachers, whose identity initialization is exact), so its cost varies across teachers; for all three procedures we report mean$\,\pm\,$std across teacher SAEs. Times exclude evaluation and the $\ell_0$ calibration rounds, each of which repeats the Fine-tune Procedure (at most three, Appendix~\ref{app:training}).
  }
  }
  \label{tab:wallclock}
  \small
  \setlength{\tabcolsep}{8pt}
  \begin{tabular}{l|ccc}
  \toprule
  Model
    & Init Procedure & Fine-tune Procedure & Total \\
  \midrule
  GPT-2 small
    & $25$\scalebox{0.7}{$\,\pm\,11$}\,s & $202$\scalebox{0.7}{$\,\pm\,120$}\,s & $227$\scalebox{0.7}{$\,\pm\,120$}\,s \\
  Pythia-160m
    & $199$\scalebox{0.7}{$\,\pm\,165$}\,s & $1100$\scalebox{0.7}{$\,\pm\,210$}\,s & $1299$\scalebox{0.7}{$\,\pm\,273$}\,s \\
  Gemma-2-2B
    & $197$\scalebox{0.7}{$\,\pm\,24$}\,s & $1323$\scalebox{0.7}{$\,\pm\,290$}\,s & $1420$\scalebox{0.7}{$\,\pm\,310$}\,s \\
  \bottomrule
  \end{tabular}\vspace{-0.1in}
\end{table}

  \noindent\textbf{Runtime and scalability.}
  Table~\ref{tab:wallclock} reports the wall-clock cost of the RSAE training procedures and shows two properties of the method: \textbf{(i)~Cheap approximate teacher reproduction.} With the rational coefficients set in closed form, the only model-specific adaptation of the $18$ coefficients and two scales to the teacher pre-activation space completes in a few minutes even on Gemma-2-2B, and brings the RSAE within a few percent of the teacher's reconstruction error before any fine-tuning. \textbf{(ii)~Lightweight fine-tune overhead.} The rational activation adds only $(p+1)+q+2=20$ scalar parameters per RSAE, negligible relative to the $\mathcal{O}(d_{\mathrm{in}}d_{\mathrm{sae}})$ encoder/decoder parameters, and its per-step cost is dominated by producing the language-model activations rather than by the SAE itself. The corresponding $10$K-step fine-tune ($20.5$M tokens, about $4\%$ of the teachers' training budget) costs $\sim$17\,min on Pythia-160m and $\sim$22\,min on Gemma-2-2B on a single RTX~5090. The $\ell_0$ calibration of Appendix~\ref{app:training} repeats this run at most three more times, so upgrading a released teacher costs between a few minutes and about $1.5$ hours.

\noindent\textbf{Sanity check on feature quality.} On the eight-task SAEBench sparse-probing panel, the RSAE's full-dictionary probe accuracy on Pythia-160m matches its ReLU teacher's ($93.71\%$ vs.\ $93.69\%$; Appendix~\ref{appx:probing}), so the gains do not degrade the learned features.

\section{Conclusion}

We introduced the \emph{Rational Sparse Autoencoder} (RSAE), an
SAE whose encoder activation is a learnable rational function, and demonstrated its advantages against three popular baseline activation families (ReLU, JumpReLU, TopK) both theoretically and empirically.
At matched expected $\ell_0$, our theory proves non-inferiority to ReLU and,
under stated assumptions, reconstruction advantages over JumpReLU and, for
large budgets, fixed raw $\mathrm{TopK}$.
On three host language models, we implemented the RSAE as an upgrade of existing pretrained SAEs and observed better fidelity at matched
sparsity: the RSAE improves on its teacher in reconstruction error, cross-entropy degradation and loss
recovered, and the relative gains grow with the teacher's $\ell_0$. Freezing the rational removes most of the gain on ReLU teachers, showing that this gain is due to the trainable activation.
Because the upgrade adds only
$20$ scalar parameters per autoencoder and a training run takes at most about $25$ minutes
 on a single consumer GPU, any released ReLU, JumpReLU, or TopK SAE
can in principle be replaced by its RSAE counterpart at small
cost.

\noindent\textbf{Limitations.} Our evaluation covers three open-weight LLMs up to Gemma-2-2B, with dictionaries of at most $2^{16}$ latents, and the three most widely released activation families; the gains on larger models and wider dictionaries remain to be measured. Our goal is reconstruction fidelity at fixed sparsity; beyond checking that sparse-probing accuracy is not degraded, we do not study how the RSAE affects feature interpretability.
Combining the rational with Gated, BatchTopK or Matryoshka SAEs is left to future work.

\subsection*{AI use statement}
In this work, we used generative AI tools only for writing and polishing the
paper, namely drafting and editing parts of the text to improve readability.
We did not use generative AI tools for any of the other tasks listed in the
ICLR 2027 AI Policy for Authors. We have reviewed all AI-assisted text, and we
take responsibility for the final content of this work, including text and
claims produced with the aid of generative AI.

\section*{Acknowledgements}
N.Y. and Y.Y. acknowledge support by the Defense Advanced Research Projects Agency (DARPA) under award HR00112590032. This research was, in part, funded by the U.S. Government by an agreement with Cornell University.

\bibliographystyle{iclr2027_conference}
\bibliography{Reference}
\appendix
\setcounter{theorem}{0}

\section{Theoretical Derivation}

\newtheorem{rsaetheorem}{Theorem}[section]
\newtheorem{rsaelemma}[rsaetheorem]{Lemma}
\newtheorem{rsaeassumption}[rsaetheorem]{Assumption}

\begingroup

\subsection{ReLU Containment}
\label{appx:relu-containment}

\begin{rsaetheorem}[Same-budget non-inferiority to ReLU]
\label{thm:relu-containment}
For numerator and denominator degree bounds $p\ge1$ and $q\ge0$, every ReLU
SAE in~\eqref{eq:sae-skel} has a clipped RSAE representation with the same
encoder, decoder, biases, and latent width, using one shared rational function
and two shared positive scalar scales.  The codes and reconstructions agree
for every input.  Consequently, for any input distribution with finite second
moment, the infimum of RSAE reconstruction error at any feasible expected
$\ell_0$ budget is no greater than the corresponding ReLU infimum under the
same encoder and decoder parameter constraints.  The same non-inferiority
holds for the regularized objective~\eqref{eq:rsae-objective}, for every
$\lambda\ge0$ and every regularizer that depends only on the code, including $\ell_1$.
\end{rsaetheorem}

\begin{proof}
Fix the encoder, decoder, and biases of a ReLU SAE, and write
$\vh(\vx)=\mWenc(\vx-\vbdec)+\vbenc$.  Choose
\begin{equation*}
  P(t)=t,\qquad Q(t)=1,\qquad r(t)=t,
  \qquad C_{\mathrm{in}}=C_{\mathrm{out}}=1.
\end{equation*}
These choices satisfy the degree bounds and give a rational function with no
real poles.  With $[\cdot]_+$ applied element-wise, the clipped RSAE code obeys
\begin{align*}
  \vz^{\mathrm{Rat}}(\vx)
  &=C_{\mathrm{out}}
    \left[r\!\left(\frac{\vh(\vx)}{C_{\mathrm{in}}}\right)\right]_+
    =[\vh(\vx)]_+=\vz^{\mathrm{ReLU}}(\vx),\\
  \widehat{\vx}^{\mathrm{Rat}}(\vx)
  &=\mWdec\vz^{\mathrm{Rat}}(\vx)+\vbdec
    =\widehat{\vx}^{\mathrm{ReLU}}(\vx).
\end{align*}
Thus reconstruction error, $\ell_0$ support size, and any regularizer of the
code (in particular its $\ell_1$ magnitude) are preserved pointwise and hence in expectation.  Every feasible ReLU
parameter choice is therefore feasible for RSAE at the identical expected
$\ell_0$ budget with the same objective value.  Taking infima proves both
non-inferiority statements.
\end{proof}

\subsection{General Same-Budget Reconstruction Advantage over Fixed Raw TopK}
\label{appx:general-theory}

Under the stability and nonlinear-calibration persistence conditions below, a
clipped RSAE code with a shared rational activation attains lower reconstruction
risk than fixed raw $\mathrm{TopK}$ at the same expected $\ell_0$ budget for
sufficiently large $k$.  The comparison uses an intermediate threshold support:
replacing the fixed top-$k$ support incurs a sublinear cost, whereas nonlinear
calibration on the threshold support yields a linear gain.

Consider a sequence of SAE instances indexed by $n$.  For each $n$, let
\begin{equation*}
  Y_n\in L^2(\Omega;\mathbb R^{d_n}),
  \qquad
  D_n=[d_{n,1},\ldots,d_{n,M_n}]\in\mathbb R^{d_n\times M_n},
  \qquad
  S_n=W_nY_n+b_n,
\end{equation*}
where $W_n\in\mathbb R^{M_n\times d_n}$,
$b_n\in\mathbb R^{M_n}$, and $D_n$ are deterministic and
$\lVert d_{n,j}\rVert_2=1$.  Fix
$k_n\in\{1,\ldots,M_n-1\}$ and assume that the coordinates of $S_n$ have no
ties and at least $k_n$ positive entries almost surely.  With descending rank,
the fixed raw $\mathrm{TopK}$ code is
\begin{equation*}
  z_{n,j}^{\mathrm T}
  =S_{n,j}\mathbf 1\!\left\{\operatorname{rank}_{S_n}(j)\le k_n\right\}.
\end{equation*}
Assume that the pooled score measure
$\sum_{j=1}^{M_n}\mathcal L(S_{n,j})$ is nonatomic and that there exists
$\tau_n>0$ for which
\begin{equation}
  A_{n,j}=\mathbf 1\{S_{n,j}>\tau_n\},
  \qquad
  N_n=\sum_{j=1}^{M_n}A_{n,j},
  \qquad
  \mathbb EN_n=k_n.
  \label{eq:theory-threshold-count}
\end{equation}
Define the raw threshold code and reconstruction risk by
\begin{equation*}
  z_{n,j}^0=S_{n,j}A_{n,j},
  \qquad
  \mathcal R_n(z)=\mathbb E\lVert Y_n-D_nz\rVert_2^2.
\end{equation*}
Thus $\mathbb E\lVert z_n^{\mathrm T}\rVert_0
=k_n=\mathbb E\lVert z_n^0\rVert_0$.
Write $[u]_+=\max\{u,0\}$.

For each $n$, let $V_{n,1}>\cdots>V_{n,M_n}$ be the ordered scores and let
$\bar d_{n,m}$ be the decoder column paired with $V_{n,m}$.  Set
\begin{equation*}
  E_{n,m-1}=Y_n-\sum_{\ell<m}\bar d_{n,\ell}V_{n,\ell},
  \qquad
  H_{n,m}=2V_{n,m}\bar d_{n,m}^{\mathsf T}E_{n,m-1}-V_{n,m}^2.
\end{equation*}

For the calibration quantities, fix an instance and suppress its index.  On the
threshold support $A_j=\mathbf 1\{S_j>\tau\}$, define the raw and shifted
active-affine responses
\begin{equation*}
  a_{\mathrm J}(s)=s,
  \qquad
  a_{\mathrm R}(s)=s-\tau,
  \qquad s>\tau,
\end{equation*}
and use their affine extensions to $\mathbb R$.  For
$q\in\{\mathrm J,\mathrm R\}$, choose
\begin{equation}
  \alpha^{q,*}\in
  \underset{\alpha\ge0}{\operatorname{argmin}}
  \ \mathbb E\left\lVert
  Y-\alpha\sum_{j=1}^M d_j a_q(S_j)A_j
  \right\rVert_2^2,
  \label{eq:theory-nnls}
\end{equation}
taking the smallest minimizer when it is nonunique.  Let
$(a_*,\alpha^*)$ be the parameters of the lower-risk fitted response, breaking
ties in favor of $a_{\mathrm J}$.  Define
\begin{equation*}
  z_j^{\mathrm A}=\alpha^*a_*(S_j)A_j,
  \qquad
  e_{\mathrm A}=Y-Dz^{\mathrm A}.
\end{equation*}

To express a shared calibration direction across active coordinates, define the
pooled active-coordinate law $\mathbb P_*$ by requiring, for every bounded
measurable $F$,
\begin{equation}
  \mathbb E_*F(J,S_J,Z_J,R_J)
  =\frac1k\sum_{j=1}^M
  \mathbb E\!\left[
    A_jF(j,S_j,z_j^{\mathrm A},d_j^{\mathsf T}e_{\mathrm A})
  \right],
  \qquad Q=Z_JR_J.
  \label{eq:theory-pooled-law}
\end{equation}
This is a probability law because $\sum_j\mathbb EA_j=k$.  For a bounded
measurable shared modulation $\phi:\mathbb R\to\mathbb R$, define
\begin{equation*}
  X_j^\phi=z_j^{\mathrm A}\phi(S_j),
  \qquad
  X^\phi=(X_1^\phi,\ldots,X_M^\phi)^{\mathsf T},
\end{equation*}
For the $n$th instance, denote the resulting law and expectation by
$\mathbb P_{*,n}$ and $\mathbb E_{*,n}$, and attach the index $n$ likewise to
$Z_J$, $Q$, and $X^\phi$.
Set
\begin{equation*}
  \ell_n(s)=\mathbb E_{*,n}[Q_n\mid S_{n,J}=s].
\end{equation*}
Whenever $0<\mathbb E_{*,n}Q_n^2<\infty$, define its correlation ratio by
\begin{equation*}
  \eta_n^2
  =\frac{\mathbb E_{*,n}\ell_n(S_{n,J})^2}
         {\mathbb E_{*,n}Q_n^2}.
\end{equation*}

\begin{rsaeassumption}[Stability]
\label{ass:stability}
There exist constants $\gamma>2$, $0\le\bar\kappa_N<\infty$,
$0\le B<\infty$, $0\le\bar\kappa_X<\infty$, and $0\le A_2<\infty$,
independent of $n$, and deterministic centers
$\lambda_n\in\mathbb R$ such that
\begin{align*}
  \operatorname{Var}(N_n)
  &\le \bar\kappa_Nk_n,\\
  \sup_{1\le m\le M_n}
  \lVert H_{n,m}-\lambda_n\rVert_{L^\gamma}
  &\le B,\\
  \mathbb E_{*,n}Z_{n,J}^2
  &\le A_2^2.
\end{align*}
For every bounded measurable real $\phi$ with
$\lVert\phi\rVert_\infty\le1$,
\begin{align*}
  \mathbb E\lVert D_nX_n^\phi\rVert_2^2
  &\le \bar\kappa_X
  \sum_{j=1}^{M_n}\mathbb E[(X_{n,j}^\phi)^2].
\end{align*}
\end{rsaeassumption}

\begin{rsaeassumption}[Nonlinear-calibration persistence]
\label{ass:persistent-calibration}
There exist constants $\sigma>0$, $C_Q\in(0,\infty)$, and $\rho\in(0,1]$,
independent of $n$, such that
\begin{align*}
  Q_n\in L^4(\mathbb P_{*,n}),
  \qquad \mathbb E_{*,n}Q_n^2
  &\ge\sigma^2,\\
  \mathbb E_{*,n}Q_n^4
  &\le C_Q\{\mathbb E_{*,n}Q_n^2\}^2,\\
  \eta_n^2
  &\ge\rho.
\end{align*}
\end{rsaeassumption}

\begin{rsaetheorem}[General same-budget reconstruction advantage]
\label{thm:general-topk-advantage}
Let $k_n\to\infty$, and suppose Assumptions~\ref{ass:stability} and
\ref{ass:persistent-calibration} hold for the sequence above.  Define
\begin{equation*}
  a_0=\frac{\sigma\rho^{3/2}}{\sqrt{C_Q}},
  \qquad
  q_0=\bar\kappa_XA_2^2,
  \qquad
  g_0=\min\!\left\{\frac{a_0}{8},\frac{a_0^2}{4q_0}\right\}.
\end{equation*}
There is a constant $c_\gamma<\infty$, depending only on $\gamma$, such that,
for every $n$ and every $\varepsilon>0$, there exist a
finite-degree shared rational function $r$ with no real poles
and a nonconstant reduced denominator, and a shared scalar
$C_{\mathrm{out}}>0$, such that, with
$C_{\mathrm{in}}=1$ in~\eqref{eq:rsae-act}, the clipped RSAE code satisfies
\begin{align*}
  z_{n,j}^{\mathrm{Rat}}
  &=C_{\mathrm{out}}[r(S_{n,j})]_+,\\
  \mathbf 1\{z_{n,j}^{\mathrm{Rat}}\ne0\}
  &=\mathbf 1\{S_{n,j}>\tau_n\}
  \quad\text{almost surely},
  \qquad j=1,\ldots,M_n,\\
  \mathbb E\lVert z_n^{\mathrm{Rat}}\rVert_0
  &=\mathbb EN_n
  =k_n
  =\mathbb E\lVert z_n^{\mathrm T}\rVert_0.
\end{align*}
The resulting code satisfies
\begin{align}
  \mathcal R_n(z_n^{\mathrm{Rat}})-\mathcal R_n(z_n^{\mathrm T})
  &\le
  \mathcal R_n(z_n^0)-\mathcal R_n(z_n^{\mathrm T})
  -g_0k_n+\varepsilon
  \notag\\
  &\le
  c_\gamma B\sqrt{\bar\kappa_Nk_n}-g_0k_n+\varepsilon.
  \label{eq:theory-main-bound}
\end{align}
Hence, for every sufficiently large $n$, such a code can be chosen to satisfy
\begin{equation}
  \mathcal R_n(z_n^{\mathrm{Rat}})
  <\mathcal R_n(z_n^{\mathrm T}).
  \label{eq:theory-strict-conclusion}
\end{equation}
\end{rsaetheorem}

The conclusion also applies along sparse sequences satisfying $k_n/M_n\to0$.

\begin{rsaelemma}[Support-replacement cost]
\label{lem:support-replacement}
Fix one instance and suppress its index.  For every $\gamma>2$, there exists
$c_\gamma<\infty$, depending only on $\gamma$, such that the raw
$\mathrm{TopK}$ and raw threshold codes defined
above satisfy, for every $\lambda\in\mathbb R$ with
$\sup_m\lVert H_m-\lambda\rVert_{L^\gamma}<\infty$,
\begin{equation}
  \left|\mathcal R(z^{\mathrm T})-\mathcal R(z^0)\right|
  \le
  c_\gamma
  \left(\sup_{1\le m\le M}
  \lVert H_m-\lambda\rVert_{L^\gamma}\right)
  \sqrt{\operatorname{Var}(N)}.
  \label{eq:support-replacement-bound}
\end{equation}
\end{rsaelemma}

\begin{proof}
For $m=0,\ldots,M$, set
\begin{equation}
  L_m=\left\lVert Y-\sum_{\ell\le m}\bar d_\ell V_\ell\right\rVert_2^2.
\end{equation}
Expansion of the square gives $L_{m-1}-L_m=H_m$.  Since the threshold support
is the prefix of random length $N$, the quantity $H_m$ is exactly the reduction
in squared reconstruction error from adding rank $m$ to the preceding prefix.
Telescoping yields
\begin{equation}
  \mathcal R(z^{\mathrm T})-\mathcal R(z^0)
  =\mathbb E\sum_{m>k}\mathbf 1\{N\ge m\}H_m
   -\mathbb E\sum_{m\le k}\mathbf 1\{N<m\}H_m.
  \label{eq:exact-support-exchange}
\end{equation}
The identity $\mathbb EN=k$ implies
\begin{equation}
  \mathbb E(N-k)_+=\mathbb E(k-N)_+.
  \label{eq:balanced-exchanges}
\end{equation}
Therefore the $\lambda$ contributions cancel in
\eqref{eq:exact-support-exchange}.  Set $\vartheta=1-1/\gamma$ and
$B_\gamma(\lambda)=\sup_m\lVert H_m-\lambda\rVert_{L^\gamma}$.
H\"older's inequality then gives
\begin{align}
  \left|\mathcal R(z^{\mathrm T})-\mathcal R(z^0)\right|
  &\le B_\gamma(\lambda)
  \sum_{r\ge1}\left\{
    \mathbb P(N-k\ge r)^\vartheta
    +\mathbb P(k-N\ge r)^\vartheta
  \right\}
  \notag\\
  &\le 2^{1/\gamma}B_\gamma(\lambda)
  \sum_{r\ge1}\mathbb P(|N-k|\ge r)^\vartheta.
  \label{eq:holder-support-bound}
\end{align}
Let $v=\sqrt{\operatorname{Var}(N)}$.  By Chebyshev's inequality,
\begin{equation}
  \sum_{r\ge1}\mathbb P(|N-k|\ge r)^\vartheta
  \le\sum_{r\ge1}\min\{1,v^2/r^2\}^{\vartheta}
  \le C_\gamma v.
  \label{eq:chebyshev-rank-sum}
\end{equation}
For $0<v\le1$, the last sum is at most
$v^{2\vartheta}\sum_{r\ge1}r^{-2\vartheta}\le
v\sum_{r\ge1}r^{-2\vartheta}$.  For $v>1$, splitting at $\lceil v\rceil$
bounds it by $2v+v/(2\vartheta-1)$.  If $v=0$, then $N=k$ almost surely and
\eqref{eq:exact-support-exchange} is zero.  Taking
$c_\gamma=2^{1/\gamma}C_\gamma$ proves the claim.
\end{proof}

\begin{rsaelemma}[Fixed-support nonlinear calibration gain]
\label{lem:rational-gain}
Fix $n$ in a sequence satisfying
Assumptions~\ref{ass:stability} and
\ref{ass:persistent-calibration}, and suppress the instance index.  For every
$\varepsilon>0$, there exist a shared scalar
$C_{\mathrm{out}}>0$ and a finite-degree shared rational function
$r$ with no real poles and a nonconstant reduced denominator such that
\begin{gather*}
  z_j^{\mathrm{Rat}}=C_{\mathrm{out}}[r(S_j)]_+,
  \qquad
  \mathbf 1\{z_j^{\mathrm{Rat}}\ne0\}
  =\mathbf 1\{S_j>\tau\}
  \quad\text{almost surely},\\
  \mathcal R(z^{\mathrm{Rat}})
  \le \mathcal R(z^{\mathrm A})-g_0k+\varepsilon
  \le \mathcal R(z^0)-g_0k+\varepsilon.
\end{gather*}
\end{rsaelemma}

\begin{proof}
The feasible choice $q=\mathrm J$ and $\alpha=1$ in
\eqref{eq:theory-nnls} recovers $z^0$.  Since $(a_*,\alpha^*)$ is the lower-risk
solution,
\begin{equation}
  \mathcal R(z^{\mathrm A})\le\mathcal R(z^0).
  \label{eq:theory-affine-risk}
\end{equation}
Assumption~\ref{ass:persistent-calibration} implies $\alpha^*>0$: otherwise
$z^{\mathrm A}=0$ and hence $Q=0$ almost surely.  The first-order condition
for~\eqref{eq:theory-nnls} therefore gives
\begin{equation}
  \sum_{j=1}^M
  \mathbb E\!\left[z_j^{\mathrm A}d_j^{\mathsf T}e_{\mathrm A}\right]=0.
  \label{eq:theory-scale-stationarity}
\end{equation}
Consequently,
$\mathbb E_*Q=\mathbb E_*\ell(S_J)=0$ under the pooled law.

We first lower-bound the score-dependent residual alignment available to a
shared response.  Define finite signed measures
\begin{equation}
  \Lambda_j(B)=
  \mathbb E\!\left[
    z_j^{\mathrm A}d_j^{\mathsf T}e_{\mathrm A}
    \mathbf 1\{S_j\in B\}
  \right],
  \qquad
  \Lambda=\sum_{j=1}^M\Lambda_j,
  \label{eq:theory-signed-measures}
\end{equation}
and set
\begin{equation}
  \mathsf S=\lVert\Lambda\rVert_{\mathrm{TV}},
  \qquad
  \mathsf Z=\sum_{j=1}^M\mathbb E[(z_j^{\mathrm A})^2].
  \label{eq:theory-STZ}
\end{equation}
Here $\mathsf S$ measures the shared signal and $\mathsf Z$ the affine-code
energy.
The pooled law and conditional expectation give
\begin{equation}
  \Lambda(B)
  =k\mathbb E_*[\ell(S_J)\mathbf 1\{S_J\in B\}],
  \qquad
  \mathsf S=k\mathbb E_*|\ell(S_J)|.
  \label{eq:theory-TV-signal}
\end{equation}
Conditional Jensen and Assumption~\ref{ass:persistent-calibration} yield
\begin{equation}
  \mathbb E_*\ell(S_J)^2
  =\eta^2\mathbb E_*Q^2
  \ge\rho\sigma^2,
  \qquad
  \mathbb E_*\ell(S_J)^4
  \le\mathbb E_*Q^4
  \le\frac{C_Q}{\rho^2}
       \{\mathbb E_*\ell(S_J)^2\}^2.
  \label{eq:theory-signal-moments}
\end{equation}
Consequently, the $L^1$--$L^2$--$L^4$ interpolation inequality gives
\begin{equation}
  \mathbb E_*|\ell(S_J)|
  \ge
  \frac{\{\mathbb E_*\ell(S_J)^2\}^{3/2}}
       {\{\mathbb E_*\ell(S_J)^4\}^{1/2}}
  \ge\frac{\sigma\rho^{3/2}}{\sqrt{C_Q}}
  =a_0.
  \label{eq:theory-signal-L1}
\end{equation}
Together with Assumption~\ref{ass:stability}
and~\eqref{eq:theory-pooled-law}, this yields
\begin{equation}
  \mathsf S\ge a_0k,
  \qquad
  \mathsf Z=k\mathbb E_*Z_J^2\le A_2^2k.
  \label{eq:theory-STZ-bounds}
\end{equation}

We convert this pooled signal into a shared rational direction.  Let
$g=d\Lambda/d|\Lambda|$ be the polar density.  Since
$|g|=1$ almost everywhere with respect to $|\Lambda|$, density of
$C_c(\mathbb R)$ in $L^1(|\Lambda|)$, followed by clipping to $[-1,1]$, gives
$h\in C_c(\mathbb R)$ such that
\begin{equation*}
  \lVert h\rVert_\infty\le1,
  \qquad
  \int|h-g|\,d|\Lambda|\le\frac{\mathsf S}{8}.
\end{equation*}
Let $\mathcal A$ be the nonunital real algebra generated by
$(1+s^2)^{-1}$ and $s(1+s^2)^{-1}$.  It is a subalgebra of
$C_0(\mathbb R)$ that separates points and vanishes nowhere.  The locally
compact Stone--Weierstrass theorem therefore makes $\mathcal A$ uniformly
dense in $C_0(\mathbb R)$.  Choose $p\in\mathcal A$ with
$\lVert p-h\rVert_\infty\le1/8$ and set $\phi=(8/9)p$.  Then
$\lVert\phi\rVert_\infty\le1$, and
\begin{equation*}
  \int\phi\,d\Lambda
  \ge\frac89\left(\mathsf S-\frac{\mathsf S}{8}
                         -\frac{\mathsf S}{8}\right)
  =\frac{2\mathsf S}{3}.
\end{equation*}
Every element of $\mathcal A$ is a bounded rational function with no real
poles.  Thus $\phi$ is rational with no real poles and satisfies
\begin{equation}
  \lVert\phi\rVert_\infty\le1,
  \qquad
  \sum_{j=1}^M\int\phi\,d\Lambda_j\ge\frac{\mathsf S}{2}.
  \label{eq:theory-rational-direction}
\end{equation}

For $0<t<1/2$, define the shared rational response
\begin{equation}
  r_t(s)=a_*(s)\{1+t\phi(s)\}.
  \label{eq:theory-rational-path}
\end{equation}
Since $1+t\phi(s)>0$, the perturbed code
$\alpha^*r_t(S_j)A_j$ is positive exactly on the threshold support.  Its
displacement from $z^{\mathrm A}$ is
\begin{equation}
  \Delta z_j(t)=tX_j^\phi.
  \label{eq:theory-code-displacement}
\end{equation}
By~\eqref{eq:theory-rational-direction},
\begin{align}
  \mathbb E[e_{\mathrm A}^{\mathsf T}D\Delta z(t)]
  &=t\sum_{j=1}^M\int\phi\,d\Lambda_j
  \notag\\
  &\ge \frac{t\mathsf S}{2}.
  \label{eq:theory-linear-gain}
\end{align}

Assumption~\ref{ass:stability} and $|\phi|\le1$ bound the quadratic term by
\begin{align}
  \mathbb E\lVert D\Delta z(t)\rVert_2^2
  &=t^2\mathbb E\lVert DX^\phi\rVert_2^2
  \notag\\
  &\le\bar\kappa_X\mathsf Zt^2.
  \label{eq:theory-quadratic-cost}
\end{align}
The exact quadratic expansion of $\mathcal R$ gives
\begin{equation}
  \mathcal R(z^{\mathrm A})-
  \mathcal R(z^{\mathrm A}+\Delta z(t))
  \ge t\mathsf S-t^2\mathsf C,
  \qquad
  \mathsf C=\bar\kappa_X\mathsf Z.
  \label{eq:theory-quadratic-gain}
\end{equation}
Since $\mathsf S>0$, Equation~\eqref{eq:theory-linear-gain} implies
$\mathbb E\lVert DX^\phi\rVert_2^2>0$.  Thus $\mathsf C>0$ and
$q_0\ge\mathsf C/k>0$.  Choose
\begin{equation}
  t_0=\min\!\left\{\frac14,\frac{\mathsf S}{2\mathsf C}\right\}.
  \label{eq:theory-step-size}
\end{equation}
Then
\begin{equation}
  \mathcal R(z^{\mathrm A})-
  \mathcal R(z^{\mathrm A}+\Delta z(t_0))
  \ge
  \min\!\left\{\frac{\mathsf S}{8},
                  \frac{\mathsf S^2}{4\mathsf C}\right\}.
  \label{eq:theory-finite-calibration-gain}
\end{equation}
Equation~\eqref{eq:theory-STZ-bounds} implies
$\mathsf C\le q_0k$ and $\mathsf S\ge a_0k$; therefore the right-hand side
of~\eqref{eq:theory-finite-calibration-gain} is at least $g_0k$.

It remains to realize the perturbed response by a clipped rational function
with exactly the threshold support.  Define
\begin{equation}
  \sigma_\tau(s)=\frac{s-\tau}{1+(s-\tau)^2},
  \qquad
  f(s)=r_{t_0}(s)\mathbf 1\{s>\tau\},
  \qquad
  \nu=\sum_{j=1}^M\mathcal L(S_j).
  \label{eq:theory-sign-carrier}
\end{equation}
Since $a_*(s)>0$ and $1+t_0\phi(s)>0$ on $(\tau,\infty)$, the function $f$ is
nonnegative.  Moreover, $S_j\in L^2$ and the boundedness of $\phi$ imply
$f\in L^2(\nu)$.
Define $u(s)=\sqrt{f(s)/\sigma_\tau(s)}$ for $s>\tau$ and $u(s)=0$ otherwise.
Let $\zeta$ be the measure defined by $d\zeta=\sigma_\tau^2d\nu$.  Then
\begin{equation}
  \lVert u\rVert_{L^4(\zeta)}^4
  =\int_{(\tau,\infty)}f(s)^2\,d\nu(s)<\infty.
  \label{eq:theory-u-integrability}
\end{equation}
First approximate $u$ in $L^4(\zeta)$ by functions in $C_c(\mathbb R)$, which
is dense in that space.  Uniform approximation by the rational algebra above
then yields finite-degree rational functions $q_m$ with no real poles such that
$q_m\to u$ in $L^4(\zeta)$.  For any positive sequence $\xi_m\downarrow0$, set
\begin{equation}
  \widetilde r_m(s)
  =\sigma_\tau(s)
   \left\{q_m(s)^2+\frac{\xi_m}{1+s^2}\right\}.
  \label{eq:theory-exact-support-rational}
\end{equation}
The term in braces is strictly positive, and hence
\begin{equation}
  \{s:\widetilde r_m(s)>0\}=(\tau,\infty).
  \label{eq:theory-rational-sign}
\end{equation}
Moreover,
\begin{equation}
  \lVert q_m^2-u^2\rVert_{L^2(\zeta)}
  \le
  \lVert q_m-u\rVert_{L^4(\zeta)}
  \bigl(\lVert q_m\rVert_{L^4(\zeta)}
        +\lVert u\rVert_{L^4(\zeta)}\bigr)
  \longrightarrow0.
  \label{eq:theory-square-convergence}
\end{equation}
The sign identity above and the triangle inequality now give
\begin{align*}
  \lVert[\widetilde r_m]_+-f\rVert_{L^2(\nu)}
  &\le \lVert q_m^2-u^2\rVert_{L^2(\zeta)} \\
  &\quad+
  \xi_m\left\lVert
    \mathbf 1_{(\tau,\infty)}
    \frac{\sigma_\tau(s)}{1+s^2}
  \right\rVert_{L^2(\nu)}
  \longrightarrow0.
\end{align*}
The second norm is finite because $\nu(\mathbb R)=M$ and its integrand is
bounded.
Each $\widetilde r_m$ is bounded, has no real poles, and changes sign at
$\tau$.  Moreover, its reduced denominator is nonconstant, since a bounded
polynomial is constant whereas $\widetilde r_m$ changes sign.

Define
\begin{equation}
  z_j^{(m)}
  =\alpha^*[\widetilde r_m(S_j)]_+.
  \label{eq:theory-approximating-code}
\end{equation}
The shared scale $\alpha^*$ is strictly positive, the support is exactly
$\{j:S_j>\tau\}$, and $Dz^{(m)}$ converges in $L^2$ to
$D(z^{\mathrm A}+\Delta z(t_0))$.  Continuity of the squared $L^2$ risk
and~\eqref{eq:theory-finite-calibration-gain} give, for all sufficiently large
$m$,
\begin{equation}
  \mathcal R(z^{(m)})
  \le\mathcal R(z^{\mathrm A})-g_0k+\varepsilon.
\end{equation}
Choosing such an $m$ and setting
$r=\widetilde r_m$,
$C_{\mathrm{out}}=\alpha^*$, and
$z^{\mathrm{Rat}}=z^{(m)}$ proves the claim.
\end{proof}

\begin{proof}[Proof of Theorem~\ref{thm:general-topk-advantage}]
Lemma~\ref{lem:support-replacement} and
Assumption~\ref{ass:stability} give
\begin{equation}
  \left|\mathcal R_n(z_n^{\mathrm T})-
  \mathcal R_n(z_n^0)\right|
  \le c_\gamma B\sqrt{\bar\kappa_Nk_n}.
  \label{eq:theory-support-rate}
\end{equation}
Lemma~\ref{lem:rational-gain} yields
\begin{align}
  \mathcal R_n(z_n^{\mathrm{Rat}})-\mathcal R_n(z_n^{\mathrm T})
  &\le
  \mathcal R_n(z_n^0)-\mathcal R_n(z_n^{\mathrm T})
  -g_0k_n+\varepsilon
  \notag\\
  &\le c_\gamma B\sqrt{\bar\kappa_Nk_n}-g_0k_n+\varepsilon,
\end{align}
which proves~\eqref{eq:theory-main-bound}.  The construction and support
identity in Lemma~\ref{lem:rational-gain} give the displayed code and exact
support claims, and~\eqref{eq:theory-threshold-count} gives the matched
expected $\ell_0$ budget.  Since $g_0>0$ and $k_n\to\infty$, choose
\begin{equation}
  0<\varepsilon_n<
  g_0k_n-c_\gamma B\sqrt{\bar\kappa_Nk_n}
\end{equation}
for every sufficiently large $n$.  Substitution into
\eqref{eq:theory-main-bound} proves~\eqref{eq:theory-strict-conclusion}.
\end{proof}

\subsection{General Same-Budget Reconstruction Advantage over JumpReLU}
\label{appx:jumprelu-advantage}

The nonlinear-calibration argument also applies to JumpReLU with per-feature
thresholds.  Let $Y\in L^2(\Omega;\mathbb R^d)$,
$S=WY+b\in\mathbb R^M$, and let $D=[d_1,\ldots,d_M]$ have unit-norm
columns.  For thresholds $\theta_j>0$, define
\begin{equation*}
  A_j=\mathbf 1\{S_j>\theta_j\},
  \qquad z_j^{\mathrm J}=S_jA_j,
  \qquad k=\sum_{j=1}^M\mathbb EA_j>0,
  \qquad \mathcal R(z)=\mathbb E\lVert Y-Dz\rVert_2^2.
\end{equation*}
Fit a shared scalar by
\begin{equation*}
  \alpha^*\in\underset{\alpha\ge0}{\operatorname{argmin}}
  \,\mathcal R(\alpha z^{\mathrm J}),
  \qquad z^{\mathrm A}=\alpha^*z^{\mathrm J},
  \qquad e_{\mathrm A}=Y-Dz^{\mathrm A},
\end{equation*}
taking the smallest minimizer when it is nonunique.  Define the pooled law
$\mathbb P_*$ by~\eqref{eq:theory-pooled-law} with these $A_j$, $z^{\mathrm A}$,
and $k$, and define $Q$, $\ell$, $\eta$, and $X^\phi$ as in
Appendix~\ref{appx:general-theory}.

\begin{rsaetheorem}[General reconstruction advantage over JumpReLU]
\label{thm:jumprelu-advantage}
Suppose the bounds on $\mathbb E_*Z_J^2$ and
$\mathbb E\lVert DX^\phi\rVert_2^2$ in Assumption~\ref{ass:stability}, and
Assumption~\ref{ass:persistent-calibration}, hold for the quantities above.
Then $q_0=\bar\kappa_XA_2^2>0$.  Define $g_0$ from these constants as in
Theorem~\ref{thm:general-topk-advantage}.
For every $\varepsilon>0$, there exist an affine encoder
$\widehat S=\widehat W Y+\widehat b$, a finite-degree shared rational function
$r$ with no real poles and a nonconstant reduced denominator, and a shared
scalar $C_{\mathrm{out}}>0$, such that, with $C_{\mathrm{in}}=1$, the clipped RSAE
code satisfies
\begin{align*}
  z_j^{\mathrm{Rat}}
  &=C_{\mathrm{out}}[r(\widehat S_j)]_+,\qquad
  \mathbf 1\{z_j^{\mathrm{Rat}}\ne0\}=A_j
  \quad\text{almost surely},\\
  \mathbb E\lVert z^{\mathrm{Rat}}\rVert_0
  &=k=\mathbb E\lVert z^{\mathrm J}\rVert_0,\\
  \mathcal R(z^{\mathrm{Rat}})
  &\le\mathcal R(z^{\mathrm J})-g_0k+\varepsilon.
\end{align*}
The decoder $D$ and latent width $M$ are unchanged, and $g_0>0$.
\end{rsaetheorem}

\begin{proof}
The feasible choice $\alpha=1$ gives
$\mathcal R(z^{\mathrm A})\le\mathcal R(z^{\mathrm J})$.
Persistence implies $\alpha^*>0$, and scalar stationarity gives
$\mathbb E_*Q=0$.  The signed-measure argument
in~\eqref{eq:theory-signed-measures}--\eqref{eq:theory-STZ-bounds} therefore
gives a pooled calibration signal of at least $a_0k$.
Construction~\eqref{eq:theory-rational-direction} yields a bounded pole-free
rational $\phi$ with $\lVert\phi\rVert_\infty\le1$; the decoded-modulation
bound then gives the linear gain and quadratic cost
in~\eqref{eq:theory-quadratic-gain}.  The positive residual alignment forces
$\mathbb E\lVert DX^\phi\rVert_2^2>0$ by Cauchy--Schwarz, hence
$\bar\kappa_X>0$.  Persistence also gives $\mathbb E_*Z_J^2>0$, so
$A_2>0$ and $q_0>0$.  The step size
$t_0\in(0,1/4]$ in~\eqref{eq:theory-step-size} gives
\begin{equation*}
  f(s)=s\{1+t_0\phi(s)\},
  \qquad \bar z_j=\alpha^*f(S_j)A_j,
  \qquad
  \mathcal R(\bar z)\le\mathcal R(z^{\mathrm A})-g_0k.
\end{equation*}
In particular, $0<f(s)\le2s$ for $s>0$.

We realize this calibrated code using a shared rational response and an affine
encoder.  Fix $\delta_0>0$.  For $R>\max_j\theta_j$ and
$0<\delta<\min_j(R-\theta_j)$, choose $c_1,\ldots,c_M>0$ recursively so that
\begin{equation*}
  c_{j+1}\delta>c_j(R-\theta_j),
  \qquad j=1,\ldots,M-1.
\end{equation*}
The intervals $I_j=[c_j\delta,c_j(R-\theta_j)]$, $j=1,\ldots,M$, are
disjoint.  On their union, set
\begin{equation*}
  F(u)=f(\theta_j+u/c_j),\qquad u\in I_j,
\end{equation*}
and extend $F$ continuously between the intervals and to zero outside a compact
subset of $(0,\infty)$, with $0\le F\le2R$.  Define the affine encoder by
\begin{equation*}
  \widehat W=\operatorname{diag}(c_1,\ldots,c_M)W,
  \qquad
  \widehat b=\operatorname{diag}(c_1,\ldots,c_M)(b-\theta),
  \qquad \widehat S_j=c_j(S_j-\theta_j),
\end{equation*}
where $\theta=(\theta_1,\ldots,\theta_M)^{\mathsf T}$.
For $\theta_j+\delta\le S_j\le R$,
$F(\widehat S_j)=f(S_j)$, while both
$F(\widehat S_j)$ and $f(S_j)A_j$ vanish when $S_j\le\theta_j$.
Using $0\le F\le2R$ and $0<f(s)\le2s$ for $s>0$ on the remaining events
gives
\begin{align*}
  \sum_{j=1}^M\mathbb E
    |F(\widehat S_j)-f(S_j)A_j|^2
  &\le16\sum_{j=1}^M
    \mathbb E[S_j^2\mathbf 1\{S_j>R\}]\\
  &\quad+16R^2\sum_{j=1}^M
    \mathbb P(\theta_j<S_j<\theta_j+\delta).
\end{align*}
Since $S_j\in L^2$, the first term tends to zero as $R\to\infty$.
For each fixed $R$, the second tends to zero as $\delta\downarrow0$ by
continuity of probability for decreasing events.  Choose $R$ and then
$\delta$ so that the displayed bound is less than $\delta_0^2/4$.

Set $\sigma_0(u)=u/(1+u^2)$ and define
$v(u)=\sqrt{F(u)/\sigma_0(u)}$ for $u>0$, with $v(u)=0$ for $u\le0$.
Because $F$ is continuous with compact support in $(0,\infty)$,
$v\in C_c(\mathbb R)$.  Uniform approximation in the rational algebra used
in Lemma~\ref{lem:rational-gain} gives bounded pole-free rational functions
$q_m$ with $q_m\to v$ uniformly.  For $\xi_m>0$ tending to zero, define
\begin{equation*}
  r_m(u)=\sigma_0(u)
  \left\{q_m(u)^2+\frac{\xi_m}{1+u^2}\right\}.
\end{equation*}
These functions are bounded and have no real poles.  The expression in braces
is strictly positive, so $\{u:r_m(u)>0\}=(0,\infty)$.  Their reduced
denominators are nonconstant: a bounded polynomial is constant, whereas
$r_m$ changes sign at zero.  Moreover,
\begin{equation*}
  \lVert[r_m]_+-F\rVert_\infty
  \le\frac12\lVert q_m-v\rVert_\infty
      (\lVert q_m\rVert_\infty+\lVert v\rVert_\infty)
      +\frac{\xi_m}{2}
  \longrightarrow0.
\end{equation*}
Choose $m$ so that this norm is less than $\delta_0/(2\sqrt M)$.
The triangle inequality in $L^2(\Omega;\mathbb R^M)$ gives
\begin{equation*}
  \left(\sum_{j=1}^M\mathbb E
    |[r_m(\widehat S_j)]_+-f(S_j)A_j|^2\right)^{1/2}
  <\delta_0.
\end{equation*}
Set $C_{\mathrm{out}}=\alpha^*$ and
$z_j^{\mathrm{Rat}}=\alpha^*[r_m(\widehat S_j)]_+$.
Since $c_j>0$, its support is exactly $\{j:S_j>\theta_j\}$.
The decoded approximation error satisfies
\begin{equation*}
  \lVert D(z^{\mathrm{Rat}}-\bar z)\rVert_{L^2}
  \le\alpha^*\lVert D\rVert_{\mathrm{op}}\delta_0.
\end{equation*}
Choosing $\delta_0$ sufficiently small makes
$\mathcal R(z^{\mathrm{Rat}})\le\mathcal R(\bar z)+\varepsilon$ by
continuity of the squared $L^2$ risk.  Combining the risk inequalities and
the support identity proves the theorem.
\end{proof}
\suppressfloats[t]

\newtheorem*{rsaerestatedtheorem}{Theorem~\ref{thm:main-reconstruction} (restated)}
\begin{samepage}
\begin{rsaerestatedtheorem}
Under Assumptions~\ref{ass:main-stability}
and~\ref{ass:main-persistent-calibration}, in the settings of
Theorems~\ref{thm:general-topk-advantage} and~\ref{thm:jumprelu-advantage},
there are constants $c\ge0$ and $g>0$, depending only on the assumption
bounds, such that for every $\varepsilon>0$ there exists an RSAE
with a finite-degree shared rational activation without real poles, the
same decoder and latent width, and
$\mathbb E\lVert z^{\mathrm{Rat}}\rVert_0=k$, satisfying
\begin{equation*}
  \mathcal R(z^{\mathrm{base}})-\mathcal R(z^{\mathrm{Rat}})
  \ge
  \begin{cases}
    gk-c\sqrt{k}-\varepsilon,
      & \mathrm{base}=\mathrm{TopK},\\
    gk-\varepsilon,
      & \mathrm{base}=\mathrm{JumpReLU}.
  \end{cases}
\end{equation*}
The encoder is unchanged for TopK; an affine reparameterization preserves
the per-feature support for JumpReLU.
\end{rsaerestatedtheorem}

\begin{proof}
For TopK, Theorem~\ref{thm:general-topk-advantage} gives the stated inequality
with $c=c_\gamma B\sqrt{\bar\kappa_N}$ and $g=g_0$.
For JumpReLU, Theorem~\ref{thm:jumprelu-advantage} gives the stated inequality
with $g=g_0$, using its corresponding assumption constants.
Both constructions preserve the decoder, latent width, and expected support
size, and use a finite-degree shared rational activation without real poles.
The TopK construction leaves the encoder unchanged; the JumpReLU construction
uses an affine reparameterization preserving each feature's support.
\end{proof}
\end{samepage}

\endgroup

\section{Detailed Empirical Results}
\label{appx:detailed-results}

\subsection{Experimental Setup}
\label{appx:implementation}

\begin{algorithm}[t]
       \caption{Rational Sparse Autoencoder (RSAE) construction.}
       \label{alg:rsae}
       \begin{algorithmic}[1]
       \Require Pre-trained teacher SAE
         $\{\widetilde{\Wb}_{\mathrm{enc}},\widetilde{\bb}_{\mathrm{enc}},\widetilde{\Wb}_{\mathrm{dec}},\widetilde{\bb}_{\mathrm{dec}}\}$
         with activation $\phi^{\mathrm{teacher}}\in\{\mathrm{ReLU},\,\mathrm{JumpReLU}_{\tilde{\theta}},\,\mathrm{TopK}_k\}$;
         rational type $(p,q)$;
         calibration batch $\{\vx_n\}_{n=1}^{N_{\mathrm{cal}}}$;
         sparsity regularizer $S$ with coefficients $\mu\ge 0$, $\lambda>0$;
         knee sharpness $K$.
       \Statex \textbf{Step 1 --- RSAE Initialization Procedure.}
         \If{$\phi^{\mathrm{teacher}}$ is JumpReLU or TopK}
           \State Absorb the thresholds into the teacher weights:
                  $\widetilde{\Wb}_{\mathrm{enc}}[:,j]\leftarrow\widetilde{\Wb}_{\mathrm{enc}}[:,j]/\tilde{\theta}_j$,
                  $\widetilde{b}_{\mathrm{enc},j}\leftarrow\widetilde{b}_{\mathrm{enc},j}/\tilde{\theta}_j-1$,
                  $\widetilde{\Wb}_{\mathrm{dec}}[j,:]\leftarrow\tilde{\theta}_j\,\widetilde{\Wb}_{\mathrm{dec}}[j,:]$
                  \Comment{every feature becomes $(v+1)\,\ones[v>0]$}
         \EndIf
         \State Initialize the encoder and decoder with
                $\{\widetilde{\Wb}_{\mathrm{enc}},\widetilde{\bb}_{\mathrm{enc}},\widetilde{\Wb}_{\mathrm{dec}},\widetilde{\bb}_{\mathrm{dec}}\}$
                and compute the teacher pre-activations $\vh_n$ on the calibration batch.
         \State Set $C_{\mathrm{in}}=C_{\mathrm{out}}=\operatorname{std}(\vh)$.
         \State Identity-init: $\ab^\ast=(0,1,0,\dots,0)$, $\bb^\ast=(1,0,\dots,0)$.
                \Comment{exactly ReLU after the clamp}
         \If{$\phi^{\mathrm{teacher}}$ is JumpReLU or TopK}
           \State Knee-init: $r_{(\ab^\ast,\bb^\ast)}(t)=\dfrac{a_1t+a_2t^2+a_3t^3}{1+K^2t^2}$ with
                  $a_3=K^2\frac{C_{\mathrm{in}}}{C_{\mathrm{out}}}$, $a_2=\frac{K^2}{C_{\mathrm{out}}}$,
                  $a_1=\max\bigl(\frac{K}{C_{\mathrm{out}}},\,1.05\,\frac{a_2^2}{4a_3}\bigr)$.
         \EndIf
         \For{each initialization of $(\ab,\bb)$}
           \Comment{identity only for ReLU; knee and identity for gated}
           \State With the encoder and decoder frozen, minimize \eqref{eq:init-distill} over
                  $(\ab,\bb,C_{\mathrm{in}},C_{\mathrm{out}})$.
         \EndFor
         \State Keep the best fit $(\widetilde{\ab},\widetilde{\bb},\widetilde{C}_{\mathrm{in}},\widetilde{C}_{\mathrm{out}})$.
       \Statex \textbf{Step 2 --- RSAE Fine-Tuning Procedure.}
         \State Initialize the encoder and decoder with the (rescaled) teacher weights and the rational with
                $(\widetilde{\ab},\widetilde{\bb},\widetilde{C}_{\mathrm{in}},\widetilde{C}_{\mathrm{out}})$;
                unfreeze
                $\Theta\coloneqq\{\Wb_{\mathrm{enc}},\bb_{\mathrm{enc}},\Wb_{\mathrm{dec}},\bb_{\mathrm{dec}},\ab,\bb,C_{\mathrm{in}},C_{\mathrm{out}}\}$.
         \Repeat
           \State Sample a mini-batch $\vx\sim\mathcal{D}$; compute
                  $\vh=\Wb_{\mathrm{enc}}(\vx-\bb_{\mathrm{dec}})+\bb_{\mathrm{enc}}$,
                  $\vz=\phi(\vh)$ via \eqref{eq:rsae-act}, and
                  $\vxhat=\Wb_{\mathrm{dec}}\vz+\bb_{\mathrm{dec}}$.
           \State Take an Adam step on $\Theta$ for \eqref{eq:rsae-objective},
                  $\|\vx-\vxhat\|_2^2+\lambda\,S(\vz)$.
                  \Comment{details in Appendix~\ref{app:training}}
         \Until{the step budget is reached}
       \Ensure RSAE with parameters $\Theta$.
       \end{algorithmic}
     \end{algorithm}

\paragraph{Models and hook points.}
We evaluate on three open-weight language models: \textbf{GPT-2 small} at layer~$6$ ($d_{\mathrm{in}}=768$), \textbf{Pythia-160m-deduped} at layer~$8$ (\texttt{blocks.8.hook\_resid\_post}, $d_{\mathrm{in}}=768$), and \textbf{Gemma-2-2B} at layer~$12$ (\texttt{blocks.12.hook\_resid\_post}, $d_{\mathrm{in}}=2304$).
For GPT-2 small each teacher is used at its own hook point: \texttt{blocks.6.hook\_resid\_pre} for the ReLU teacher and \texttt{blocks.6.hook\_resid\_post} for the TopK teacher.
Gemma-2-2B is run in \texttt{bf16}, as in SAEBench; GPT-2 small and Pythia-160m are run in \texttt{fp32}.

\paragraph{Teacher SAE checkpoints.}
Teacher SAEs are pre-trained, publicly released checkpoints. For GPT-2 small the ReLU teacher is \texttt{gpt2-small-res-jb}~\citep{bloom2024gpt2sae} ($d_{\mathrm{sae}}=24{,}576$) and the TopK teacher is the OpenAI-v5 \texttt{gpt2-small-resid-post-v5-32k}~\citep{gao2025scaling} ($d_{\mathrm{sae}}=32{,}768$, with layer-norm preprocessing of the residual stream). For Pythia-160m, the ReLU, JumpReLU and TopK teachers are taken from the SAEBench~\citep{karvonen2025saebench} releases \texttt{adamkarvonen/\allowbreak saebench\_\allowbreak pythia-160m-deduped\_\allowbreak width-2pow\{12,14,16\}\_\allowbreak date-0108}, which provide six trainers (sparsity levels) per family at each width $d_{\mathrm{sae}}\in\{2^{12},2^{14},2^{16}\}$; the main results use $d_{\mathrm{sae}}=2^{12}$. For Gemma-2-2B all three teacher families are taken from the SAEBench release \texttt{adamkarvonen/saebench\_gemma-2-2b\_width-2pow12\_date-0108} ($d_{\mathrm{sae}}=4{,}096$, the only released width).

\paragraph{Activation buffer and calibration.}
Training activations are streamed through TransformerLens hooks. For Pythia-160m and Gemma-2-2B they come from the Pile (\texttt{monology/pile-uncopyrighted}) and reproduce the SAEBench activation buffer the teachers were trained with (contexts of $1024$ tokens with the first token dropped, batches of $2048$ tokens drawn from a shuffled buffer); for GPT-2 small they come from OpenWebText. The first $4{,}096$ tokens of the stream form the calibration buffer: it sets the shared input and output scales $C_{\mathrm{in}}=C_{\mathrm{out}}$ in \eqref{eq:rsae-act} to the standard deviation of the teacher's pre-activations, and the activation distillation step \eqref{eq:init-distill} is fitted on it (on a $2{,}048$-token subsample for gated teachers). The fine-tuning procedure then trains on $10{,}000$ fresh batches ($\approx 20.5$M tokens), with inputs divided by the dataset norm factor as in SAEBench.

\paragraph{Hardware.}
All RSAE training reported in Table~\ref{tab:wallclock} is run on a single NVIDIA RTX 5090 (32~GB). End-to-end wall-clock per (model, teacher) is dominated by the fine-tuning procedure. Runs at $d_{\mathrm{sae}}=2^{12}$ on Pythia-160m were executed two at a time on the same GPU, so their wall-clock times are upper bounds.

\paragraph{Evaluation protocol.}
All metrics are computed with the SAEBench core evaluation~\citep{karvonen2025saebench}, applied identically to every teacher and every RSAE, on OpenWebText documents packed into $128$-token windows that each start with a BOS token (special-token positions are excluded from the averages). For Pythia-160m and Gemma-2-2B this evaluation text is disjoint from the training text. Reconstruction MSE (squared error summed over $d_{\mathrm{in}}$, averaged over tokens), $\ell_0$ (the number of exactly non-zero latents per token) and the alive fraction (latents with non-zero activation density) are computed on $4{,}096{,}000$ tokens. Cross-entropy degradation $\Delta\mathrm{CE} = \mathrm{CE}_{\mathrm{SAE}} - \mathrm{CE}_{\mathrm{clean}}$ and the loss-recovered fraction $\mathrm{LR} = (\mathrm{CE}_{\mathrm{zero}} - \mathrm{CE}_{\mathrm{SAE}})/(\mathrm{CE}_{\mathrm{zero}} - \mathrm{CE}_{\mathrm{clean}})$ are computed on $409{,}600$ tokens ($3{,}200$ windows), by replacing the residual stream at the SAE's host layer with $\hat{\vx}$ and reading the language model's next-token cross-entropy at every position; $\mathrm{CE}_{\mathrm{zero}}$ is the cross-entropy obtained when the residual stream at the host layer is replaced by the zero vector.

\paragraph{Sparse-probing setup (Table~\ref{tab:probing}).}
We use the SAEBench probing suite of $8$ binary-classification tasks. Per dataset we train an $\ell_2$-regularized logistic regression on the full SAE feature vector (no $k$-sparsity constraint) using SAEBench's default split, and report mean test accuracy across the dataset's evaluation seeds.

\subsection{Implementation Details of the RSAE Initialization and Training}\label{app:training}

\paragraph{Threshold re-gauging.}
\label{app:gauge}
A JumpReLU teacher computes $\phi^{\mathrm{teacher}}_j(\vh)=h_j\,\ones[h_j>\tilde{\theta}_j]$ with
a learned per-feature threshold $\tilde{\theta}_j>0$; for a TopK teacher $\tilde{\theta}_j\equiv\tilde{\theta}$
is the released inference threshold (main text, footnote). Since the RSAE activation is shared across
features, it cannot express $d_{\mathrm{sae}}$ different jump heights, so feature $j$ is rescaled by
$\tilde{\theta}_j$ before initialization:
$\widetilde{\Wb}_{\mathrm{enc}}[:,j]\leftarrow\widetilde{\Wb}_{\mathrm{enc}}[:,j]/\tilde{\theta}_j$,
$\widetilde{b}_{\mathrm{enc},j}\leftarrow\widetilde{b}_{\mathrm{enc},j}/\tilde{\theta}_j-1$,
$\widetilde{\Wb}_{\mathrm{dec}}[j,:]\leftarrow\tilde{\theta}_j\,\widetilde{\Wb}_{\mathrm{dec}}[j,:]$.
Writing $v_j=(h_j-\tilde{\theta}_j)/\tilde{\theta}_j$ for the rescaled pre-activation, the teacher's
contribution to the reconstruction is unchanged,
$h_j\ones[h_j>\tilde{\theta}_j]\,\widetilde{\Wb}_{\mathrm{dec}}[j,:]
=\bigl[(v_j+1)\ones[v_j>0]\bigr]\bigl(\tilde{\theta}_j\widetilde{\Wb}_{\mathrm{dec}}[j,:]\bigr)$,
and every feature now computes the same unit jump $\varphi(v)=(v+1)\,\ones[v>0]$. Thresholds are
floored at $0.1\cdot\mathrm{median}_j\tilde{\theta}_j$ so that dead features with
$\tilde{\theta}_j\approx0$ do not blow up the encoder scale. After the rescaling the decoder rows are no
longer unit-norm, so for re-gauged teachers the decoder is not re-normalized during fine-tuning.
Because the threshold-gated TopK teacher matches the exact top-$k$ teacher in mean $\ell_0$ but not
token by token, the Step~1 gap is slightly larger for TopK than for JumpReLU teachers.

\paragraph{Knee coefficients.}
With $t=v/C_{\mathrm{in}}$, the tail of $r(t)=(a_1t+a_2t^2+a_3t^3)/(1+K^2t^2)$ is
$r(t)=\tfrac{a_3}{K^2}t+\tfrac{a_2}{K^2}+O(1/t)$, so condition (i) of the main text,
$C_{\mathrm{out}}\,r(v/C_{\mathrm{in}})=v+1+O(1/v)$, gives $a_3=K^2C_{\mathrm{in}}/C_{\mathrm{out}}$ and
$a_2=K^2/C_{\mathrm{out}}$. Near the origin $r(t)\approx a_1t$ reaches the tail intercept
$1/C_{\mathrm{out}}$ at $t\approx1/(a_1C_{\mathrm{out}})$, so condition (ii), a knee of width $\approx1/K$,
is $a_1\ge K/C_{\mathrm{out}}$. For condition (iii) note that $r(t)=t\,(a_1+a_2t+a_3t^2)/(1+K^2t^2)$ with a
positive denominator, so $r(t)<0$ for all $t<0$ iff the quadratic has no real root, $a_2^2<4a_1a_3$;
$a_1=\max(K/C_{\mathrm{out}},\,1.05\,a_2^2/(4a_3))$ satisfies both with a $5\%$ margin. In the notation of
\eqref{eq:rsae-act}, $\ab^\ast=(0,a_1,a_2,a_3,0,\dots,0)$ and $\bb^\ast=(1,0,K^2,0,\dots,0)$, for which
the safe denominator reduces to $1+K^2t^2$. Sharper knees ($K\ge120$) fit the jump more closely after
distillation but drift during fine-tuning (the reconstruction error rises during the learning-rate
warm-up and does not recover) and were not used. Because the knee is not optimal for very sparse
teachers, Step~1 runs the distillation \eqref{eq:init-distill} from both the identity and the knee
initialization and keeps the lower objective; the identity start wins at $\ell_0\approx20$ and the knee
everywhere else.

\paragraph{Architecture.}
All experiments use a shared rational of type $(p,q)=(9,8)$ with the safe denominator of
\eqref{eq:rsae-act}, so the activation adds $18$ coefficients and two scales to the teacher's parameter
count. The scales are parameterized as $C_{\mathrm{in}}=e^{\gamma_{\mathrm{in}}}$,
$C_{\mathrm{out}}=e^{\gamma_{\mathrm{out}}}$ with $\gamma$ trainable. The encoder input convention
follows the teacher: SAEBench Standard and TopK SAEs subtract $\widetilde{\bb}_{\mathrm{dec}}$ before the
encoder, the SAEBench JumpReLU SAEs do not.

\paragraph{Decoder-norm folding.}
SAEBench Standard (ReLU) SAEs are released with non-unit decoder rows (row norms $0.45$--$0.74$ for
Pythia-160m). Since fine-tuning re-normalizes $\Wb_{\mathrm{dec}}$ after every step, we first apply the
exact reparameterization $\widetilde{\Wb}_{\mathrm{dec}}[j,:]\leftarrow\widetilde{\Wb}_{\mathrm{dec}}[j,:]/n_j$,
$\widetilde{\Wb}_{\mathrm{enc}}[:,j]\leftarrow n_j\widetilde{\Wb}_{\mathrm{enc}}[:,j]$,
$\widetilde{b}_{\mathrm{enc},j}\leftarrow n_j\widetilde{b}_{\mathrm{enc},j}$ with
$n_j=\|\widetilde{\Wb}_{\mathrm{dec}}[j,:]\|_2$, which is valid for any positively homogeneous activation
(ReLU and the identity-initialized rational). Without it the first fine-tuning step moves the model far
from the teacher (reconstruction error $0.058\to0.168$ on the Pythia-160m ReLU teacher with $\ell_0\approx160$ at $d_{\mathrm{sae}}=2^{14}$).

\paragraph{Sparsity regularizer.}
We use a target-$\ell_0$ regularizer
$S(\vz)=\bigl(\mathbb{E}[\widehat{\ell}_0]/(\rho\,\ell_0^{\mathrm{teacher}})-1\bigr)^{2}$, where
$\ell_0^{\mathrm{teacher}}$ is the teacher's mean number of active features, $\rho$ a target fraction and
$\widehat{\ell}_0$ a differentiable count of the RSAE's support $\{j:P(h_j/C_{\mathrm{in}})>0\}$; no
threshold parameter exists, so $S$ acts on sparsity only through the rational coefficients, the scales
and the encoder. For ReLU teachers, whose rational is close to linear at its root, we count on the
activation, $\widehat{\ell}_0=\sum_j\sigma\bigl(C_{\mathrm{out}}\,r_{(\ab,\bb)}(h_j/C_{\mathrm{in}})/\varepsilon\bigr)$
with $\sigma$ the logistic function and $\varepsilon=10^{-3}$ in normalized activation units. For gated
teachers the rational is a steep knee whose inactive side stays within $\varepsilon$ of zero over a wide
range of $h$, which inflates this count (soft $190$--$209$ for a hard $\ell_0$ of $160$) and
over-sparsifies the model early in training; we therefore count in pre-activation space,
$\widehat{\ell}_0=\sum_j\sigma\bigl((h_j-\theta_{\vh})/\varepsilon\bigr)$ with
$\theta_{\vh}=C_{\mathrm{in}}u^\ast$ and $u^\ast$ the first upward zero crossing of $P$ on $[-4,4]$,
located by bisection and differentiated through the implicit function theorem,
$\partial u^\ast/\partial a_k=-(u^\ast)^k/P'(u^\ast)$, implemented as one Newton correction with
automatic differentiation. This form is complemented by the penalty
$\varepsilon^{-1}\sum_j[C_{\mathrm{out}}r(h_j/C_{\mathrm{in}})]_+\,\ones[h_j<\theta_{\vh}-4\varepsilon]$,
which forbids a second positive lobe of $P$ far from the root (without it one run activated almost every
feature on strongly negative pre-activations while the soft count still read the target). In
\eqref{eq:init-distill} the activation-space count with $\rho=1$ and $\mu=1$ is used; in
\eqref{eq:rsae-objective} $\lambda=1$. An $\ell_1$ penalty, $S(\vz)=\|\vz\|_1$, recovers the classical
objective but, for any coefficient, only retraces the teacher's Pareto front in our experiments.

\paragraph{Calibration of $\rho$.}
The evaluated $\ell_0$ of a trained RSAE lands close to, but not exactly at,
$\rho\,\ell_0^{\mathrm{teacher}}$. We therefore treat $\rho$
as the single tuned hyper-parameter: starting from $\rho=1.02$ (TopK) or $1.00$ (JumpReLU, ReLU), we
train, evaluate $\ell_0$ under the protocol below, and if $|\ell_0/\ell_0^{\mathrm{teacher}}-1|>0.01$
retrain with $\rho\leftarrow\rho\,\ell_0^{\mathrm{teacher}}/\ell_0$, for at most three rounds with
$\rho\in[0.85,1.15]$ and degenerate runs ($\ell_0/\ell_0^{\mathrm{teacher}}\notin[0.5,1.5]$) ignored. All
reported RSAEs use at most $1\%$ more active latents than their teacher, and most are within $1\%$ of the teacher's $\ell_0$. The exceptions are all sparser than their teacher: by up to $5\%$ for ReLU teachers at $d_{\mathrm{sae}}\in\{2^{14},2^{16}\}$ and for the GPT-2 small teachers, and by one latent for the TopK teachers at $\ell_0\approx80$ (Tables~\ref{tab:main-results} and~\ref{tab:saebench-teachers-widths}).

\paragraph{Step~1 distillation.}
Equation~\eqref{eq:init-distill} is minimized with Adam (learning rate $10^{-2}$) for $6{,}000$ steps on
$2{,}048$ calibration tokens drawn from the training stream, never from the evaluation data, with the
encoder and decoder frozen; when the objective has not improved for $300$ steps the optimizer is
restarted from the best iterate with the learning rate multiplied by $0.3$. For ReLU teachers the
objective is zero at the identity initialization and the step is skipped. After Step~1 the RSAE is within
$1$--$3\%$ of the teacher's reconstruction error for all gated teachers considered.

\paragraph{Step~2 optimization.}
Training data reproduce the teachers' own pipeline (the \texttt{dictionary\_learning} activation
buffer): the Pile streamed in order, contexts of $1{,}024$ tokens with the first token dropped, a shuffle
buffer of $244$ contexts refreshed when half consumed, batches of $2{,}048$ activations, and dataset
normalization $\vx\leftarrow\vx/\sqrt{\mathbb{E}\|\vx\|_2^2}$ estimated on the first $100$ batches; the
first $128$ documents are held out. We run $10{,}000$ Adam steps ($20.5$M tokens) with learning rate
$3\times10^{-4}$ (rational coefficients $3\times10^{-3}$ for gated teachers, whose knee coefficients are
$O(10^2$--$10^3)$), linear warm-up over $500$ steps and linear decay from step $8{,}000$, gradient-norm
clipping at $1$, and, for ReLU teachers, projection of the decoder gradient onto the tangent space of the
unit-norm constraint followed by row re-normalization after each step. The reconstruction term is the
per-token sum of squared errors in normalized units. An auxiliary dead-feature loss in the style of
AuxK~\citep{gao2025scaling} is added during training only: features that have not fired for $100$ steps
reconstruct the current residual through a linear read-out of their $d_{\mathrm{sae}}/32$ largest
pre-activations, weighted by $1/32$. It leaves the model unchanged and keeps the fraction of alive
features at or above the teacher's in all settings except ReLU teachers at $d_{\mathrm{sae}}=2^{16}$,
where it stays $2$--$33$ points below (Table~\ref{tab:saebench-teachers-widths}); without it, wide dense settings ($d_{\mathrm{sae}}=2^{16}$,
$\ell_0\ge640$) collapse to $8$--$37\%$ alive features. For the densest JumpReLU teachers, whose released
models have fewer than $26\%$ alive features, the auxiliary loss is disabled and training is extended to
$20{,}000$ steps.

\subsection{Sparse Probing}
\label{appx:probing}

As a sanity check that the reconstruction gains do not degrade the learned features, Table~\ref{tab:probing} reports full-dictionary linear-probe accuracy on the eight-task SAEBench panel for the Pythia-160m ReLU teacher with $\ell_0\approx160$ at $d_{\mathrm{sae}}=2^{14}$ and the RSAE fine-tuned from it. The two are on par: every per-task difference lies within $\pm0.3$ percentage points, comparable to the seed-to-seed noise of the probe, and the panel means are $93.71\%$ (RSAE) and $93.69\%$ (teacher). We make no claim that the RSAE improves interpretability; the check only indicates that its reconstruction gains are not obtained at the expense of probe accuracy.
\begin{table}[t]
  \centering
  \caption{{\small\textbf{Sparse-probing sanity check} on Pythia-160m
  ReLU teacher with $\ell_0\approx160$ at $d_{\mathrm{sae}}=2^{14}$ and the RSAE fine-tuned from it,
  on the full SAEBench panel of $8$
  binary-classification tasks.  We report full-dictionary probe
  accuracy per dataset (all SAE features, no $k$-sparsity
  constraint); higher is better.}}
  \label{tab:probing}
  \small
  \setlength{\tabcolsep}{8pt}
  \begin{tabular}{l|cc|c}
  \toprule
  Dataset
    & Teacher full $\uparrow$ & RSAE full $\uparrow$ & $\Delta$ \\
  \midrule
  \texttt{bias\_in\_bios\_set1}    & $95.28$          & $95.30$ & $+0.02$ \\
  \texttt{bias\_in\_bios\_set2}    & $93.12$ & $93.04$          & $-0.08$ \\
  \texttt{bias\_in\_bios\_set3}    & $90.94$          & $91.22$ & $+0.28$ \\
  \texttt{amazon\_reviews}         & $87.98$          & $88.00$ & $+0.02$ \\
  \texttt{amazon\_sentiment}       & $91.40$          & $91.40$          & $0.00$ \\
  \texttt{github-code}             & $96.64$ & $96.50$          & $-0.14$ \\
  \texttt{ag\_news}                & $94.18$          & $94.25$ & $+0.07$ \\
  \texttt{europarl}                & $99.96$          & $99.96$          & $0.00$ \\
  \midrule
  mean                             & $93.69$          & $93.71$ & $+0.02$ \\
  \bottomrule
  \end{tabular}%
\end{table}

\subsection{Detailed Empirical Results}

\definecolor{gainc}{HTML}{1B7F3B}
\definecolor{lossc}{HTML}{B22222}
\begin{table}[hpt]
  \centering
  \caption{\textbf{GPT-2 small (layer~6): teacher SAEs and the RSAEs fine-tuned from them}, SAEBench core evaluation on OpenWebText. The ReLU teacher is Bloom's \texttt{gpt2-small-res-jb} ($d_\mathrm{sae}=24{,}576$, hook \texttt{resid\_pre}); the TopK teacher is OpenAI's \texttt{gpt2-small-resid-post-v5-32k} ($d_\mathrm{sae}=32{,}768$, $k=32$, hook \texttt{resid\_post}); no JumpReLU SAE has been released for GPT-2 small. Each RSAE is calibrated to its teacher's $\ell_0$ (within $1\%$). MSE is the mean per-token squared reconstruction error multiplied by 100; alive and LR are in percent; $\Delta$CE is the cross-entropy increase when the reconstruction replaces the residual stream; KL is the per-token KL divergence between the model's next-token distributions with and without the reconstruction. \textbf{Bold} marks the better value within each pair; the small coloured numbers give the RSAE's change relative to its teacher (relative change for MSE, $\Delta$CE and KL; percentage points for LR), \textcolor{gainc}{green} when better and \textcolor{lossc}{red} when worse. The alive fraction is reported for completeness only; it is not an optimisation target.}
  \label{tab:gpt2}
  \small
  \setlength{\tabcolsep}{4pt}
  \begin{tabular}{l r c c c c c}
  \toprule
  SAE & $\ell_0$ & MSE$\downarrow$ & alive$\uparrow$ & $\Delta$CE$\downarrow$ & KL$\downarrow$ & LR$\uparrow$ \\
  \midrule
  ReLU teacher & 51 & 21.42 & 100.0 & 0.142 & 0.1654 & 98.6 \\
  \rowcolor{aliceblue} RSAE (ReLU-init) & 52 & \textbf{20.78} {\color{gainc}\scriptsize(-3.0\%)} & 92.0 & \textbf{0.101} {\color{gainc}\scriptsize(-28.9\%)} & \textbf{0.1239} {\color{gainc}\scriptsize(-25.1\%)} & \textbf{99.0} {\color{gainc}\scriptsize(+0.4)} \\
  \midrule
  TopK teacher & 32 & 29.44 & 100.0 & 0.114 & 0.1339 & 98.9 \\
  \rowcolor{aliceblue} RSAE (TopK-init) & 32 & \textbf{26.70} {\color{gainc}\scriptsize(-9.3\%)} & 99.2 & \textbf{0.096} {\color{gainc}\scriptsize(-16.0\%)} & \textbf{0.1168} {\color{gainc}\scriptsize(-12.8\%)} & \textbf{99.0} {\color{gainc}\scriptsize(+0.2)} \\
  \bottomrule
  \end{tabular}
\end{table}

\begin{table*}[t]
  \centering
  \caption{\textbf{SAE Performance on Pythia-160m-deduped (layer~8) across three dictionary widths ($d_\mathrm{sae} = \{2^{12}, 2^{14}, 2^{16}\}$), grouped by sparsity level ($\ell_0 = \{20, 40, 80, 160, 320, 640\}$).} Columns: $\ell_0$, the mean number of non-zero latents per token; MSE, the mean per-token squared reconstruction error $\|\xb - \hat{\xb}\|_2^2$ \textbf{multiplied by 100}; alive, the percentage of latents that fire at least once on the evaluation set; $\Delta \mathrm{CE}$, the increase in next-token cross-entropy when the reconstruction replaces the residual stream; LR, the loss-recovered percentage relative to zero ablation.\textbf{Bold} marks the best value among the architectures within each (width, level) group.}
  \label{tab:saebench-teachers-widths}
  \scriptsize
  \setlength{\tabcolsep}{2.0pt}
  \begin{tabular}{l l r c c c c r c c c c r c c c c}
  \toprule
   & & \multicolumn{5}{c}{4k ($2^{12}$)} & \multicolumn{5}{c}{16k ($2^{14}$)} & \multicolumn{5}{c}{65k ($2^{16}$)} \\
  \cmidrule(lr){3-7} \cmidrule(lr){8-12} \cmidrule(lr){13-17}
  $\ell_0$ level & SAE & $\ell_0$ & MSE$\downarrow$ & alive$\uparrow$ & $\Delta$CE$\downarrow$ & LR$\uparrow$ & $\ell_0$ & MSE$\downarrow$ & alive$\uparrow$ & $\Delta$CE$\downarrow$ & LR$\uparrow$ & $\ell_0$ & MSE$\downarrow$ & alive$\uparrow$ & $\Delta$CE$\downarrow$ & LR$\uparrow$ \\
  \midrule
  \multirow{2}{*}{$\sim$20} & JumpReLU & 19 & 3.34 & 62.2 & 0.579 & 92.2 & 19 & 2.95 & 40.9 & 0.450 & 93.9 & 19 & 2.69 & 22.0 & 0.378 & 94.9 \\
   & TopK & 20 & 3.11 & 99.8 & 0.521 & 93.0 & 20 & 2.63 & \textbf{98.6} & \textbf{0.353} & \textbf{95.2} & 20 & \textbf{2.35} & \textbf{90.3} & \textbf{0.276} & \textbf{96.3} \\
   \rowcolor{aliceblue}  & RSAE (JumpReLU-init) & 19 & 3.32 & 62.3 & 0.578 & 92.2 & 19 & 2.92 & 41.3 & 0.454 & 93.9 & 19 & 2.68 & 22.7 & 0.380 & 94.9 \\
   \rowcolor{aliceblue}  & RSAE (TopK-init) & 20 & \textbf{3.09} & \textbf{100.0} & \textbf{0.515} & \textbf{93.1} & 20 & \textbf{2.61} & 97.0 & 0.361 & 95.1 & 20 & 2.36 & 70.5 & 0.287 & 96.1 \\
  \midrule
  \multirow{3}{*}{$\sim$40} & ReLU & 44 & 3.68 & 74.8 & 0.728 & 90.2 & 55 & 3.07 & 73.1 & 0.488 & 93.4 & 59 & 2.68 & 66.9 & 0.341 & 95.4 \\
   & JumpReLU & 38 & 2.62 & 75.9 & 0.365 & 95.1 & 38 & 2.30 & 52.2 & 0.280 & 96.2 & 36 & 2.15 & 27.0 & 0.241 & 96.8 \\
   & TopK & 40 & 2.50 & 99.9 & 0.339 & 95.4 & 40 & 2.08 & 99.1 & 0.235 & 96.8 & 40 & 1.82 & \textbf{87.6} & \textbf{0.170} & \textbf{97.7} \\
   \rowcolor{aliceblue}  & RSAE (ReLU-init) & 45 & 2.75 & 98.3 & 0.421 & 94.3 & 53 & 2.26 & 92.4 & 0.323 & 95.6 & 60 & 1.87 & 65.0 & 0.193 & 97.4 \\
   \rowcolor{aliceblue}  & RSAE (JumpReLU-init) & 38 & 2.62 & 96.9 & 0.364 & 95.1 & 38 & 2.28 & 52.6 & 0.282 & 96.2 & 36 & 2.12 & 28.1 & 0.241 & 96.7 \\
   \rowcolor{aliceblue}  & RSAE (TopK-init) & 40 & \textbf{2.47} & \textbf{100.0} & \textbf{0.327} & \textbf{95.6} & 40 & \textbf{2.05} & \textbf{99.4} & \textbf{0.231} & \textbf{96.9} & 40 & \textbf{1.80} & 87.1 & 0.176 & 97.6 \\
  \midrule
  \multirow{3}{*}{$\sim$80} & ReLU & 79 & 3.04 & 75.9 & 0.526 & 92.9 & 102 & 2.51 & 74.0 & 0.341 & 95.4 & 115 & 2.18 & 68.7 & 0.246 & 96.7 \\
   & JumpReLU & 75 & 1.98 & 87.9 & 0.221 & 97.0 & 68 & 1.85 & 51.3 & 0.195 & 97.4 & 70 & 1.71 & 24.2 & 0.168 & 97.7 \\
   & TopK & 80 & 1.91 & \textbf{100.0} & 0.222 & 97.0 & 80 & 1.59 & 99.3 & 0.154 & 97.9 & 80 & 1.41 & 79.6 & 0.114 & \textbf{98.5} \\
   \rowcolor{aliceblue}  & RSAE (ReLU-init) & 79 & 2.12 & 98.7 & 0.276 & 96.3 & 104 & 1.82 & 98.4 & 0.220 & 97.0 & 115 & 1.42 & 65.6 & 0.117 & 98.4 \\
   \rowcolor{aliceblue}  & RSAE (JumpReLU-init) & 76 & 1.96 & 99.9 & 0.219 & 97.0 & 68 & 1.83 & 51.8 & 0.192 & 97.4 & 70 & 1.69 & 24.9 & 0.165 & 97.8 \\
   \rowcolor{aliceblue}  & RSAE (TopK-init) & 79 & \textbf{1.88} & \textbf{100.0} & \textbf{0.206} & \textbf{97.2} & 80 & \textbf{1.57} & \textbf{100.0} & \textbf{0.148} & \textbf{98.0} & 79 & \textbf{1.39} & \textbf{80.9} & \textbf{0.111} & \textbf{98.5} \\
  \midrule
  \multirow{3}{*}{$\sim$160} & ReLU & 122 & 2.60 & 76.2 & 0.404 & 94.5 & 164 & 2.12 & 74.4 & 0.258 & 96.5 & 192 & 1.84 & \textbf{69.3} & 0.183 & 97.5 \\
   & JumpReLU & 155 & 1.34 & 81.0 & 0.116 & 98.4 & 150 & 1.26 & 41.5 & 0.107 & 98.6 & 125 & 1.35 & 17.5 & 0.120 & 98.4 \\
   & TopK & 160 & 1.30 & 99.5 & 0.125 & 98.3 & 160 & 1.10 & 94.6 & 0.093 & 98.8 & 160 & 0.99 & 66.1 & 0.071 & 99.0 \\
   \rowcolor{aliceblue}  & RSAE (ReLU-init) & 121 & 1.69 & 99.9 & 0.187 & 97.5 & 160 & 1.38 & 98.9 & 0.141 & 98.1 & 185 & 1.08 & 64.5 & 0.072 & 99.0 \\
   \rowcolor{aliceblue}  & RSAE (JumpReLU-init) & 157 & 1.31 & \textbf{100.0} & 0.116 & 98.4 & 150 & 1.24 & 41.6 & 0.104 & 98.6 & 124 & 1.32 & 17.9 & 0.117 & 98.4 \\
   \rowcolor{aliceblue}  & RSAE (TopK-init) & 161 & \textbf{1.24} & 99.9 & \textbf{0.105} & \textbf{98.6} & 160 & \textbf{1.08} & \textbf{100.0} & \textbf{0.083} & \textbf{98.9} & 160 & \textbf{0.96} & 68.6 & \textbf{0.068} & \textbf{99.1} \\
  \midrule
  \multirow{3}{*}{$\sim$320} & ReLU & 224 & 1.98 & 76.3 & 0.259 & 96.5 & 319 & 1.60 & 74.7 & 0.165 & 97.8 & 410 & 1.37 & \textbf{70.0} & 0.111 & 98.5 \\
   & JumpReLU & 326 & 0.56 & 63.7 & 0.033 & 99.5 & 314 & 0.60 & 26.1 & 0.036 & 99.5 & 291 & 0.69 & 8.7 & 0.044 & 99.4 \\
   & TopK & 320 & 0.54 & 66.3 & 0.038 & 99.5 & 320 & 0.52 & 46.3 & 0.036 & 99.5 & 320 & 0.51 & 25.5 & 0.034 & 99.5 \\
   \rowcolor{aliceblue}  & RSAE (ReLU-init) & 224 & 1.08 & 99.2 & 0.097 & 98.7 & 322 & 0.80 & 99.2 & 0.064 & 99.1 & 409 & 0.56 & 60.0 & 0.030 & \textbf{99.6} \\
   \rowcolor{aliceblue}  & RSAE (JumpReLU-init) & 326 & 0.54 & 98.8 & 0.033 & \textbf{99.6} & 314 & 0.57 & 25.7 & 0.035 & 99.5 & 292 & 0.65 & 8.7 & 0.043 & 99.4 \\
   \rowcolor{aliceblue}  & RSAE (TopK-init) & 320 & \textbf{0.50} & \textbf{99.3} & \textbf{0.032} & \textbf{99.6} & 320 & \textbf{0.50} & \textbf{99.9} & \textbf{0.029} & \textbf{99.6} & 320 & \textbf{0.47} & 26.5 & \textbf{0.027} & \textbf{99.6} \\
  \midrule
  \multirow{3}{*}{$\sim$640} & ReLU & 336 & 1.55 & 76.4 & 0.176 & 97.6 & 501 & 1.26 & 74.9 & 0.114 & 98.5 & 682 & 1.07 & \textbf{70.4} & 0.076 & 99.0 \\
   & JumpReLU & 626 & \textbf{0.12} & 25.6 & \textbf{0.007} & \textbf{99.9} & 607 & \textbf{0.13} & 7.0 & 0.008 & \textbf{99.9} & 614 & 0.26 & 1.7 & 0.010 & \textbf{99.9} \\
   & TopK & 640 & 0.15 & 23.3 & 0.010 & \textbf{99.9} & 639 & 0.15 & 6.8 & 0.009 & \textbf{99.9} & 640 & 0.16 & 2.6 & 0.009 & \textbf{99.9} \\
   \rowcolor{aliceblue}  & RSAE (ReLU-init) & 336 & 0.58 & 96.6 & 0.045 & 99.4 & 502 & 0.31 & \textbf{83.0} & 0.018 & 99.8 & 648 & 0.14 & 37.1 & \textbf{0.005} & \textbf{99.9} \\
   \rowcolor{aliceblue}  & RSAE (JumpReLU-init) & 623 & \textbf{0.12} & 25.4 & \textbf{0.007} & \textbf{99.9} & 607 & \textbf{0.13} & 6.9 & 0.007 & \textbf{99.9} & 614 & 0.13 & 1.7 & 0.007 & \textbf{99.9} \\
   \rowcolor{aliceblue}  & RSAE (TopK-init) & 639 & 0.13 & \textbf{100.0} & \textbf{0.007} & \textbf{99.9} & 637 & \textbf{0.13} & 7.1 & \textbf{0.006} & \textbf{99.9} & 640 & \textbf{0.12} & 3.1 & 0.006 & \textbf{99.9} \\
  \bottomrule
  \end{tabular}
  \end{table*}

\begin{table*}[t]
  \centering
  \caption{\textbf{SAE performance on Gemma-2-2B (layer~12) across three dictionary widths ($d_\mathrm{sae}=\{2^{12},2^{14},2^{16}\}$), grouped by sparsity level.} SAEBench core evaluation; each RSAE is calibrated to its teacher's $\ell_0$ (within $1\%$). Columns: $\ell_0$; MSE, the mean per-token squared reconstruction error (raw units, not scaled); alive, percentage of latents that fire on the evaluation set; $\Delta$CE, the cross-entropy increase when the reconstruction replaces the residual stream (computed by SAEBench in \texttt{bf16}, hence quantised to steps of $1/64$); LR, loss recovered in percent. \textbf{Bold} marks the best value within each (width, level) group.}
  \label{tab:gemma-full}
  \scriptsize
  \setlength{\tabcolsep}{1.6pt}
  \begin{tabular}{l l r c c c c r c c c c r c c c c}
  \toprule
   & & \multicolumn{5}{c}{4k ($2^{12}$)} & \multicolumn{5}{c}{16k ($2^{14}$)} & \multicolumn{5}{c}{65k ($2^{16}$)} \\
  \cmidrule(lr){3-7} \cmidrule(lr){8-12} \cmidrule(lr){13-17}
  $\ell_0$ level & SAE & $\ell_0$ & MSE$\downarrow$ & alive$\uparrow$ & $\Delta$CE$\downarrow$ & LR$\uparrow$ & $\ell_0$ & MSE$\downarrow$ & alive$\uparrow$ & $\Delta$CE$\downarrow$ & LR$\uparrow$ & $\ell_0$ & MSE$\downarrow$ & alive$\uparrow$ & $\Delta$CE$\downarrow$ & LR$\uparrow$ \\
  \midrule
   \multirow{2}{*}{$\sim$20} & JumpReLU & 21 & 2.745 & 89.0 & \textbf{0.516} & \textbf{94.6} & 21 & 2.369 & 71.9 & \textbf{0.344} & \textbf{96.4} & 22 & \textbf{2.097} & 50.8 & \textbf{0.250} & \textbf{97.4} \\
    & TopK & 20 & 2.758 & 99.8 & 0.547 & 94.2 & 20 & \textbf{2.350} & \textbf{97.3} & \textbf{0.344} & \textbf{96.4} & 20 & 2.123 & \textbf{87.3} & \textbf{0.250} & \textbf{97.4} \\
   \rowcolor{aliceblue}  & RSAE (JumpReLU-init) & 21 & \textbf{2.744} & 89.0 & \textbf{0.516} & \textbf{94.6} & 22 & 2.369 & 71.6 & \textbf{0.344} & \textbf{96.4} & 22 & 2.115 & 50.5 & \textbf{0.250} & \textbf{97.4} \\
   \rowcolor{aliceblue}  & RSAE (TopK-init) & 20 & 2.749 & \textbf{99.9} & 0.531 & 94.4 & 20 & 2.358 & 93.9 & \textbf{0.344} & \textbf{96.4} & 20 & 2.122 & 69.2 & \textbf{0.250} & \textbf{97.4} \\
  \midrule
   \multirow{3}{*}{$\sim$40} & ReLU & 46 & 2.875 & 77.9 & 0.750 & 92.1 & 58 & 2.346 & 76.7 & 0.406 & 95.7 & 64 & 1.988 & 69.6 & 0.250 & 97.4 \\
    & JumpReLU & 42 & 2.373 & 95.0 & 0.328 & 96.5 & 43 & 1.996 & 79.4 & \textbf{0.203} & \textbf{97.9} & 43 & 1.751 & 58.4 & 0.156 & 98.4 \\
    & TopK & 40 & 2.405 & 99.8 & 0.328 & 96.5 & 40 & 1.994 & 98.5 & \textbf{0.203} & \textbf{97.9} & 40 & 1.750 & 85.5 & 0.141 & 98.5 \\
   \rowcolor{aliceblue}  & RSAE (ReLU-init) & 46 & 2.474 & 77.9 & 0.406 & 95.7 & 58 & \textbf{1.956} & 76.6 & \textbf{0.203} & \textbf{97.9} & 64 & \textbf{1.640} & 69.0 & \textbf{0.125} & \textbf{98.7} \\
   \rowcolor{aliceblue}  & RSAE (JumpReLU-init) & 42 & \textbf{2.371} & 95.1 & \textbf{0.312} & \textbf{96.7} & 43 & 1.994 & 79.4 & \textbf{0.203} & \textbf{97.9} & 43 & 1.759 & 58.6 & 0.156 & 98.4 \\
   \rowcolor{aliceblue}  & RSAE (TopK-init) & 40 & 2.390 & \textbf{99.9} & \textbf{0.312} & \textbf{96.7} & 40 & 1.996 & \textbf{98.8} & \textbf{0.203} & \textbf{97.9} & 40 & 1.740 & \textbf{85.8} & 0.141 & 98.5 \\
  \midrule
   \multirow{3}{*}{$\sim$80} & ReLU & 85 & 2.547 & 78.6 & 0.469 & 95.1 & 111 & 2.006 & 77.7 & 0.250 & 97.4 & 127 & 1.664 & 71.8 & 0.156 & 98.4 \\
    & JumpReLU & 84 & 2.004 & 99.6 & \textbf{0.188} & \textbf{98.0} & 86 & 1.634 & 89.5 & 0.125 & 98.7 & 86 & 1.418 & 66.8 & 0.094 & 99.0 \\
    & TopK & 80 & 2.067 & 99.9 & 0.219 & 97.7 & 80 & 1.664 & 99.3 & 0.125 & 98.7 & 80 & 1.416 & 89.8 & 0.094 & 99.0 \\
   \rowcolor{aliceblue}  & RSAE (ReLU-init) & 85 & 2.150 & 78.6 & 0.234 & 97.5 & 111 & \textbf{1.630} & 77.7 & \textbf{0.109} & \textbf{98.8} & 127 & \textbf{1.329} & 71.2 & \textbf{0.078} & \textbf{99.2} \\
   \rowcolor{aliceblue}  & RSAE (JumpReLU-init) & 84 & \textbf{2.000} & 99.6 & \textbf{0.188} & \textbf{98.0} & 86 & \textbf{1.630} & 89.5 & 0.125 & 98.7 & 86 & 1.422 & 67.2 & 0.094 & 99.0 \\
   \rowcolor{aliceblue}  & RSAE (TopK-init) & 80 & 2.044 & \textbf{100.0} & \textbf{0.188} & \textbf{98.0} & 80 & 1.660 & \textbf{99.7} & 0.125 & 98.7 & 80 & 1.404 & \textbf{94.4} & 0.094 & 99.0 \\
  \midrule
   \multirow{3}{*}{$\sim$160} & ReLU & 138 & 2.303 & 79.0 & 0.344 & 96.4 & 186 & 1.763 & 78.1 & 0.172 & 98.2 & 215 & 1.438 & 73.0 & 0.109 & 98.8 \\
    & JumpReLU & 170 & 1.617 & 99.6 & \textbf{0.109} & \textbf{98.8} & 173 & 1.281 & 98.1 & \textbf{0.062} & \textbf{99.3} & 163 & 1.146 & 71.7 & \textbf{0.062} & \textbf{99.3} \\
    & TopK & 160 & 1.706 & 99.9 & 0.125 & 98.7 & 160 & 1.350 & 99.4 & 0.078 & 99.2 & 160 & 1.142 & 90.4 & \textbf{0.062} & \textbf{99.3} \\
   \rowcolor{aliceblue}  & RSAE (ReLU-init) & 138 & 1.877 & 79.2 & 0.156 & 98.4 & 186 & 1.379 & 78.1 & 0.078 & 99.2 & 216 & \textbf{1.098} & 72.5 & \textbf{0.062} & \textbf{99.3} \\
   \rowcolor{aliceblue}  & RSAE (JumpReLU-init) & 171 & \textbf{1.614} & 99.6 & \textbf{0.109} & \textbf{98.8} & 174 & \textbf{1.275} & 98.2 & \textbf{0.062} & \textbf{99.3} & 164 & 1.146 & 73.1 & \textbf{0.062} & \textbf{99.3} \\
   \rowcolor{aliceblue}  & RSAE (TopK-init) & 160 & 1.673 & \textbf{100.0} & \textbf{0.109} & \textbf{98.8} & 160 & 1.338 & \textbf{99.7} & 0.078 & 99.2 & 160 & 1.126 & \textbf{95.7} & \textbf{0.062} & \textbf{99.3} \\
  \midrule
   \multirow{3}{*}{$\sim$320} & ReLU & 293 & 1.902 & 79.1 & 0.203 & 97.9 & 401 & 1.397 & 78.6 & 0.109 & 98.8 & 474 & 1.120 & 74.2 & 0.078 & 99.2 \\
    & JumpReLU & 348 & 1.155 & 99.9 & \textbf{0.062} & \textbf{99.3} & 355 & 0.911 & 99.3 & \textbf{0.047} & \textbf{99.5} & 307 & 0.887 & 71.0 & \textbf{0.031} & \textbf{99.7} \\
    & TopK & 320 & 1.301 & \textbf{100.0} & \textbf{0.062} & \textbf{99.3} & 320 & 1.020 & 98.9 & \textbf{0.047} & \textbf{99.5} & 320 & 0.879 & 86.6 & \textbf{0.031} & \textbf{99.7} \\
   \rowcolor{aliceblue}  & RSAE (ReLU-init) & 293 & 1.399 & 79.5 & 0.078 & 99.2 & 401 & 0.974 & 78.5 & \textbf{0.047} & \textbf{99.5} & 476 & \textbf{0.772} & 73.4 & \textbf{0.031} & \textbf{99.7} \\
   \rowcolor{aliceblue}  & RSAE (JumpReLU-init) & 350 & \textbf{1.151} & 99.9 & \textbf{0.062} & \textbf{99.3} & 357 & \textbf{0.908} & \textbf{99.4} & \textbf{0.047} & \textbf{99.5} & 306 & 0.894 & 71.2 & \textbf{0.031} & \textbf{99.7} \\
   \rowcolor{aliceblue}  & RSAE (TopK-init) & 320 & 1.259 & \textbf{100.0} & \textbf{0.062} & \textbf{99.3} & 320 & 1.001 & 99.3 & \textbf{0.047} & \textbf{99.5} & 320 & 0.875 & \textbf{96.9} & \textbf{0.031} & \textbf{99.7} \\
  \midrule
   \multirow{3}{*}{$\sim$640} & ReLU & 492 & 1.565 & 79.2 & 0.141 & 98.5 & 680 & 1.130 & 78.8 & 0.078 & 99.2 & 842 & 0.899 & 75.0 & 0.047 & 99.5 \\
    & JumpReLU & 693 & 0.590 & \textbf{99.9} & \textbf{0.016} & \textbf{99.8} & 712 & 0.468 & \textbf{96.2} & \textbf{0.016} & \textbf{99.8} & 607 & 0.558 & 58.2 & \textbf{0.016} & \textbf{99.8} \\
    & TopK & 640 & 0.793 & 97.3 & 0.031 & 99.7 & 640 & 0.646 & 93.2 & 0.031 & 99.7 & 640 & 0.584 & 70.8 & \textbf{0.016} & \textbf{99.8} \\
   \rowcolor{aliceblue}  & RSAE (ReLU-init) & 492 & 1.018 & 79.3 & 0.047 & 99.5 & 680 & 0.647 & 78.8 & 0.031 & 99.7 & 840 & \textbf{0.485} & 73.6 & \textbf{0.016} & \textbf{99.8} \\
   \rowcolor{aliceblue}  & RSAE (JumpReLU-init) & 698 & \textbf{0.586} & \textbf{99.9} & \textbf{0.016} & \textbf{99.8} & 717 & \textbf{0.464} & 96.1 & \textbf{0.016} & \textbf{99.8} & 607 & 0.561 & 58.4 & \textbf{0.016} & \textbf{99.8} \\
   \rowcolor{aliceblue}  & RSAE (TopK-init) & 641 & 0.754 & 97.7 & 0.031 & 99.7 & 641 & 0.622 & 94.7 & 0.031 & 99.7 & 641 & 0.573 & \textbf{91.0} & \textbf{0.016} & \textbf{99.8} \\
  \midrule
   \multirow{1}{*}{$\sim$1000} & ReLU & 702 & 1.284 & 79.2 & 0.094 & 99.0 & 981 & 0.927 & \textbf{78.7} & 0.062 & 99.3 & 1272 & 0.743 & \textbf{75.2} & 0.031 & 99.7 \\
   \rowcolor{aliceblue}  & RSAE (ReLU-init) & 702 & \textbf{0.738} & \textbf{80.1} & \textbf{0.031} & \textbf{99.7} & 981 & \textbf{0.378} & \textbf{78.7} & \textbf{0.016} & \textbf{99.8} & 1274 & \textbf{0.229} & 72.1 & \textbf{0.016} & \textbf{99.8} \\
  \bottomrule
  \end{tabular}
\end{table*}

\definecolor{gainc}{HTML}{1B7F3B}
\definecolor{lossc}{HTML}{B22222}
\begin{table}[t]
  \centering
  \caption{\textbf{GPT-2 small, ReLU teacher SAEs at every layer.} Teachers: Bloom's \texttt{gpt2-small-res-jb} ($d_\mathrm{sae}=24576$, \texttt{blocks.$\ell$.hook\_resid\_pre}). SAEBench core evaluation on OpenWebText (128-token windows); each RSAE is calibrated to its teacher's $\ell_0$ (within $1\%$). Parentheses: relative change of the RSAE vs.\ its teacher.}
  \label{tab:gpt2-layers}
  \small
  \setlength{\tabcolsep}{3pt}
  \begin{tabular}{r r c c c c}
  \toprule
  layer & $\ell_0$ & MSE$\downarrow$ teacher $\to$ RSAE & $\Delta$CE$\downarrow$ teacher $\to$ RSAE & KL$\downarrow$ teacher $\to$ RSAE & LR$\uparrow$ teacher $\to$ RSAE \\
  \midrule
  0 & 13 & 0.00014 $\to$ \textbf{0.00011} {\color{gainc}\scriptsize(-21.8\%)} & 0.021 $\to$ \textbf{0.011} {\color{gainc}\scriptsize(-46.3\%)} & 0.024 $\to$ \textbf{0.012} {\color{gainc}\scriptsize(-47.5\%)} & 99.9 $\to$ 99.9 \\
  1 & 15 & 0.027 $\to$ \textbf{0.022} {\color{gainc}\scriptsize(-18.4\%)} & 0.019 $\to$ \textbf{0.016} {\color{gainc}\scriptsize(-16.6\%)} & 0.024 $\to$ \textbf{0.021} {\color{gainc}\scriptsize(-11.5\%)} & 99.6 $\to$ 99.6 \\
  2 & 19 & 0.039 $\to$ \textbf{0.033} {\color{gainc}\scriptsize(-15.0\%)} & 0.066 $\to$ \textbf{0.036} {\color{gainc}\scriptsize(-45.7\%)} & 0.078 $\to$ \textbf{0.044} {\color{gainc}\scriptsize(-43.4\%)} & 98.6 $\to$ 99.2 \\
  3 & 26 & 0.067 $\to$ \textbf{0.065} {\color{gainc}\scriptsize(-4.1\%)} & 0.162 $\to$ \textbf{0.059} {\color{gainc}\scriptsize(-63.8\%)} & 0.178 $\to$ \textbf{0.080} {\color{gainc}\scriptsize(-55.2\%)} & 98.4 $\to$ 99.4 \\
  4 & 34 & 0.108 $\to$ \textbf{0.107} {\color{gainc}\scriptsize(-1.8\%)} & 0.124 $\to$ \textbf{0.073} {\color{gainc}\scriptsize(-41.1\%)} & 0.147 $\to$ \textbf{0.097} {\color{gainc}\scriptsize(-33.9\%)} & 99.1 $\to$ 99.5 \\
  5 & 44 & 0.152 $\to$ \textbf{0.144} {\color{gainc}\scriptsize(-5.0\%)} & 0.141 $\to$ \textbf{0.093} {\color{gainc}\scriptsize(-33.5\%)} & 0.163 $\to$ \textbf{0.117} {\color{gainc}\scriptsize(-28.2\%)} & 98.9 $\to$ 99.3 \\
  6 & 52 & 0.214 $\to$ \textbf{0.208} {\color{gainc}\scriptsize(-3.0\%)} & 0.142 $\to$ \textbf{0.101} {\color{gainc}\scriptsize(-28.9\%)} & 0.165 $\to$ \textbf{0.124} {\color{gainc}\scriptsize(-25.1\%)} & 98.6 $\to$ 99.0 \\
  7 & 59 & 0.299 $\to$ \textbf{0.277} {\color{gainc}\scriptsize(-7.4\%)} & 0.153 $\to$ \textbf{0.103} {\color{gainc}\scriptsize(-32.2\%)} & 0.178 $\to$ \textbf{0.125} {\color{gainc}\scriptsize(-29.8\%)} & 98.5 $\to$ 99.0 \\
  8 & 67 & 0.428 $\to$ \textbf{0.384} {\color{gainc}\scriptsize(-10.3\%)} & 0.156 $\to$ \textbf{0.102} {\color{gainc}\scriptsize(-34.6\%)} & 0.181 $\to$ \textbf{0.123} {\color{gainc}\scriptsize(-32.1\%)} & 98.0 $\to$ 98.7 \\
  9 & 72 & 0.633 $\to$ \textbf{0.553} {\color{gainc}\scriptsize(-12.8\%)} & 0.162 $\to$ \textbf{0.106} {\color{gainc}\scriptsize(-34.4\%)} & 0.187 $\to$ \textbf{0.128} {\color{gainc}\scriptsize(-31.7\%)} & 98.8 $\to$ 99.2 \\
  10 & 55 & 1.114 $\to$ \textbf{0.936} {\color{gainc}\scriptsize(-16.0\%)} & 0.211 $\to$ \textbf{0.143} {\color{gainc}\scriptsize(-32.3\%)} & 0.238 $\to$ \textbf{0.164} {\color{gainc}\scriptsize(-31.2\%)} & 97.4 $\to$ 98.3 \\
  11 & 60 & 1.757 $\to$ \textbf{1.461} {\color{gainc}\scriptsize(-16.8\%)} & 0.243 $\to$ \textbf{0.168} {\color{gainc}\scriptsize(-30.8\%)} & 0.265 $\to$ \textbf{0.186} {\color{gainc}\scriptsize(-29.8\%)} & 96.3 $\to$ 97.4 \\
  \bottomrule
  \end{tabular}
\end{table}

\end{document}